\documentclass[twocolumn]{article}
\usepackage{arxiv}

\usepackage{microtype}
\usepackage{graphicx}
\usepackage{subfigure}
\usepackage{booktabs} % for professional tables
\usepackage{amsmath}
\usepackage{amsthm}
\usepackage{color}
\usepackage{soul}
\usepackage{xr-hyper}
\usepackage{natbib}
\usepackage{algorithm}
\usepackage{algorithmic}
\usepackage{appendix}

\usepackage{hyperref}

\usepackage{tabularx}

\renewcommand{\shorttitle}{Fast Estimation Method for the Stability of Ensemble Feature Selectors}

\hypersetup{
pdftitle={Fast Estimation Method for the Stability of Ensemble Feature Selectors},
pdfsubject={cs.LG},
pdfauthor={Rina Onda, Zhengyan Gao, Masaaki Kotera, Kenta Oono},
pdfkeywords={Machine learning, Feature selection, Ensemble algorithm},
}

\newcommand{\nensemble}{m_{\mathrm{ensemble}}}
\newcommand{\nstability}{m_{\mathrm{stability}}}
\newcommand{\nfeatures}{n_{\mathrm{feature}}}
\newcommand{\ntarget}{n_{\mathrm{target}}}
\newcommand{\nmeaningful}{n_{\mathrm{useful}}}
\newcommand{\nmeaningfulverification}{n_{\mathrm{useful}}^{\mathrm{v}}}
\newcommand{\parameterizedsimulator}{f_{\nmeaningful, \probability}}
\newcommand{\probability}{p}
\newcommand{\nprobability}{k_p}
\newcommand{\threshold}{t_\mathrm{uniform}}
\newcommand{\rankm}{\mathrm{rank}_m}
\newcommand{\ntrees}{n_\mathrm{tree}}

\newtheorem{theorem}{Theorem}

\newcommand{\nf}{n_\mathrm{f}}
\newcommand{\nm}{n_\mathrm{m}}
\newcommand{\nt}{n_\mathrm{t}}

\begin{document}

\twocolumn[
% \icmltitle{Fast Estimation Method for the Stability of Ensemble Feature Selectors}
\title{Fast Estimation Method for the Stability of \\ Ensemble Feature Selectors}




\author{
    Rina Onda\thanks{
        Work performed during the internship at Preferred Networks.
	} \\
	Department of Physics, Graduate School of Science\\
	The University of Tokyo\\
	Tokyo, Japan \\
	\texttt{onda@icepp.s.u-tokyo.ac.jp} \\
	%% examples of more authors
	\And
	Zhengyan Gao \\
	Preferred Networks, Inc.\\
	Tokyo, Japan \\
	\texttt{zhengyan@preferred.jp} \\
	\AND
	Masaaki Kotera \\
	Preferred Networks, Inc.\\
	Tokyo, Japan \\
	\texttt{kotera@preferred.jp} \\
	\And
	Kenta Oono\thanks{
	    To whom correspondence should be addressed.
	} \\
	Preferred Networks, Inc.\\
	Tokyo, Japan \\
	\texttt{oono@perferred.jp} \\
}

% \begin{icmlauthorlist}
% \icmlauthor{Rina Onda}{ut}
% \icmlauthor{Zhengyan Gao}{pfn}
% \icmlauthor{Masaaki Kotera}{pfn}
% \icmlauthor{Kenta Oono}{pfn}
% \end{icmlauthorlist}

% \icmlaffiliation{ut}{Department of Physics, Graduate School of Science \\ The University of Tokyo}
% \icmlaffiliation{pfn}{Preferred Networks, Inc.}

% \icmlcorrespondingauthor{Kenta Oono}{oono@preferred.jp}


% \vskip 0.3in
% ]



\maketitle

\begin{abstract}
It is preferred that feature selectors be \textit{stable} for better interpretabity and robust prediction.
Ensembling is known to be effective for improving the stability of feature selectors.
Since ensembling is time-consuming, it is desirable to reduce the computational cost to estimate the stability of the ensemble feature selectors.
We propose a simulator of a feature selector, and apply it to a fast estimation of the stability of ensemble feature selectors.
To the best of our knowledge, this is the first study that estimates the stability of ensemble feature selectors and reduces the computation time theoretically and empirically.
% This document provides a basic paper template and submission guidelines.
% Abstracts must be a single paragraph, ideally between 4--6 sentences long.
% Gross violations will trigger corrections at the camera-ready phase.
\end{abstract}

\vskip 0.3in
]

\section{Introduction}
\vskip -0.3in

Feature selection is a generic term for selecting useful features, from a set of features, for machine learning tasks.
Removing noisy features avoids the curse of dimensionality and improves the prediction performance~\citep{KHAIRE2019}.
An important benefit of feature selection is to promote interpretability prediction results.
That is, feature selection provides a better understanding of the underlying process that generated the data~\citep{guyon2003introduction}.
One approach for improving the interpretability of feature selection is to make the feature selection algorithm \textit{stable} or, in other words, robust~\citep{dunne2002solutions,kalousis2005stability,kalousis2007stability}.
Generally, feature selection results vary owing to the variability of the feature selection algorithm and dataset~\citep{doi:https://doi.org/10.1002/9781118617151.ch14}.
The term \textit{stable} means that the feature selection results do not vary significantly by such randomness.
It is reported that the lack of the stability of the feature selection algorithm makes it difficult to interpret its results, especially when we handle high-dimensional data such as gene expression measurements~\citep{Ein-Dor5923,davis2006reliable}.

\textit{Ensemble feature selection} is an often-used approach to improve the stability of feature selectors~\citep{10.1007/978-3-540-87481-2_21}.
The idea of ensemble feature selection is similar to that of ensemble learning.
An ensemble feature selector integrates the results of multiple feature selectors, which we call \textit{weak} selectors in this paper.
For a feature to be selected by the final ensemble feature selector, it needs to be selected by most of the weak selectors.
Therefore, like in ensemble learning, ensemble feature selectors are expected to be stable to the noise inherent in weak selectors and dataset.
\textit{Bagging} is an example ensemble algorithm. Given a single feature selector, bagging makes the copies of the selector, feeds a sub-sampled dataset to each copy, and aggregates the results.

Measuring the stability of ensemble feature selectors in advance is challenging.
Thus far, in order to calculate it, we have had to actually use them many times.
However, ensemble feature selectors are computationally expensive.
If an ensemble feature selector uses $m$ weak feature selectors, it is typically $m$ times slower than a single feature selector.
The computation costs of standard feature selection algorithms, such as random forest importance score and the minimum redundancy maximum relevance (mRMR) feature selection~\citep{ding2005minimum}, are high.
This problem becomes even more pronounced when we ensemble such costly algorithms.
To the best of our knowledge, no studies have focused on the fast computation of stability.

Another problem with ensemble selection is that it only provides a qualitative account of how an ensemble makes the final feature selector stable.
Indeed, many studies reported that an ensemble increases the stability of feature selectors~\citep{10.1007/978-3-540-87481-2_21,doi:https://doi.org/10.1002/9781118617151.ch14,pes2020ensemble}.
However, to the best of our knowledge, no study has quantitatively analyzed the conditions under which an ensemble is effective.
Because of this, we cannot decide whether an ensemble helps improve the stability of the feature selector until we try it many times.
This is a major constraint on the practical application of ensemble selectors.

To address these problems, we propose a fast simulation-based method for estimating the stability of ensemble selectors.
We construct a \textit{simulated} feature selector that emulates a real selector using two parameters, i.e., the number of features ($\nmeaningful$) that the feature selector prefers to select, and the probability ($\probability$), which reflects the uncertainty of the feature selectors and the dataset.
The main advantage of this method is that we can estimate the stability \textit{before} applying real ensemble selectors to real data. 
Normally, to calculate the stability of an ensemble selector, it needs to be run $O(\nstability \times \nensemble)$ times.
Here, $\nstability$ is the number of copies of the feature selector for calculating its stability, and $\nensemble$ is the number of weak selectors for constructing an ensemble selector.
Our proposed method only needs to run the feature selector $O(\nstability + \nensemble)$ times.
These two explainable parameters enable us to understand in what way the characteristics of the dataset and feature selector affect the stability of the ensemble selector.
We applied the proposed method to three different cancer datasets and confirmed the validity and computational efficiency of our proposed method.

\section{Method}

\subsection{Problem Settings}

Consider the problem to select $\ntarget$ features from the feature set $S$ consisting of $\nfeatures$ features.
We assign indices to the feature set as $S=\{1,\ldots , \nfeatures\}$.
A dataset $D$, a base feature selector $f$, and an ensemble algorithm $\mathcal{E}$ are provided.
We construct $\nensemble$ weak selectors from $f$, and build the final ensemble selector using $\mathcal{E}$ by integrating the results of the weak selectors.
The goal of the task is to estimate the stability of the ensemble selector.

The feature subsets that are selected by weak selectors can vary depending on the algorithmic randomness.
For instance, in bagging, each weak selector is obtained by applying the base selector to the sub-dataset that is randomly sampled from the dataset $D$. 
The algorithm randomness is observed also in feature selectors themselves, such as the importance scores of random forest.

We impose the following assumptions on the problem --- (1) The feature selector $f$ ranks (i.e., gives a total order to) the features in $S$, (2) The ensemble algorithm $\mathcal{E}$ takes an arbitrary number of ranked feature sets, and outputs a subset of features consisting of $\ntarget$ features, and (3) The stability computation takes an arbitrary number of feature subsets, each of which consists of $\ntarget$ features, as an input, and outputs a real value.
The concrete implementation is described in Section~\ref{sec:implementation}.

\begin{algorithm}[t]
    \caption{Ensemble feature selector stability estimation}
    \hspace*{\algorithmicindent} \textbf{Require:} Base feature selector $f$. Dataset $D$. \\ 
    \hspace*{\algorithmicindent} \phantom{\textbf{Require:}} Ensemble algorithm $\mathcal{E}$.\\
    \hspace*{\algorithmicindent} \textbf{Output:} Estimated stability value $J$ for the ensemble\\ 
    \hspace*{\algorithmicindent} \phantom{\textbf{Output:}} feature selector of $f$.      
    
    \begin{algorithmic}[1]\label{alg:simulated-ensemble-model}
        \STATE Determine the threshold $\threshold$ using the uniform feature selector (Section~\ref{sec:n-meaningful-estimation-method}).
        \STATE Estimate $\nmeaningful$ from $\threshold$ (Section~\ref{sec:n-meaningful-estimation-method}).
        \STATE Calculate the stability of $f$ using $D$.
        \STATE Estimate $\probability$ using the stability of $f$ (Section~\ref{sec:p-estimation-method}).
        \STATE (Optional) Verify $\nmeaningful$ (Section~\ref{sec:n-meaningful-verification}).
        \FOR{$j = 1, \ldots, \nstability$}
        \STATE Run Algorithm~\ref{alg:single-simulated-ensemble-model} to obtain the feature subset $s_j$ of a simulated ensemble feature selector.
        \ENDFOR
        \STATE Compute $J$ from $\{s_j \mid j = 1, \ldots, \nstability\}$.
    \end{algorithmic}
\end{algorithm}
\begin{algorithm}[t]
    \caption{Simulated ensemble feature selector}
    \hspace*{\algorithmicindent} \textbf{Require:} Simulator parameters $\nmeaningful$ and $\probability$. \\
    \hspace*{\algorithmicindent} \phantom{\textbf{Require:}} Ensemble algorithm $\mathcal{E}$.\\
    \hspace*{\algorithmicindent} \textbf{Output:} Feature selection result $s_j (\subset S)$ of a single\\
    \hspace*{\algorithmicindent} \phantom{\textbf{Output:}} simulated ensemble feature selector.
    
    \begin{algorithmic}[1]\label{alg:single-simulated-ensemble-model}
        \FOR{$m = 1, \ldots, \nensemble$}
            \STATE Generate the feature subset $S_m$ from $S'=\{1, \ldots, \nmeaningful\}$  (Section~\ref{sec:s-m-configuration}). 
            \STATE Run Algorithm~\ref{alg:simulated-feature-selector} to compute the feature rank $\rankm$ of the $m$-th weak feature selector $f_m$ constructed from $p$ and $S_m$ (Section~\ref{sec:simulated-feature-selector-modeling}).
        \ENDFOR
        \STATE Select $s_j$ by applying the ensemble algorithm $\mathcal{E}$ to $\{\rankm \mid m=1, \ldots, \nensemble\}$.
    \end{algorithmic}
\end{algorithm}

\begin{algorithm}[t]
    \caption{Simulated feature selector}
    \hspace*{\algorithmicindent} \textbf{Require:} Parameter $\probability\in [0, 1]$. Feature set $S_m (\subset S)$.\\
    \hspace*{\algorithmicindent} \textbf{Output:} Feature rank $\rankm$.

    \begin{algorithmic}[1]\label{alg:simulated-feature-selector}
        \FOR{$i = 1, \ldots , \nfeatures$}
            \STATE Sample $r$ from the uniform distribution over $[0, 1]$.
            \IF{($r < \probability$ \AND $S_m = \emptyset$) \OR ($r \geq \probability$ \AND $S\setminus S_m \not = \emptyset$)}
                \STATE Choose $s$ from $S\setminus S_m$ uniformly randomly.
            \ELSE
                \STATE Choose $s$ from $S_m$ uniformly randomly.
            \ENDIF
            \STATE $\rankm[s] = i$.
            \STATE Remove $s$ from $S$ (Either $S_m$ or $S\setminus S_m$ changes).
        \ENDFOR
    \end{algorithmic}
\end{algorithm}

\subsection{Model}\label{sec:models}

\subsubsection{Overview}\label{sec:model-overview}

To solve this problem, we propose a simulation-based estimation method.
The idea is to build a feature selector simulator that mimics the behavior of the base selector and to estimate the stability using the simulated ensemble feature selector instead of the real one.
Algorithm~\ref{alg:simulated-ensemble-model} shows the pseudocode for the proposed method,
which uses the simulated selector as elaborated in Algorithm~\ref{alg:single-simulated-ensemble-model}.
The algorithm constructs simulated selectors $f_1, \ldots , f_{\nensemble}$ that model the base selector $f$ (as well as the dataset $D$).
Then, it calculates the stability quickly by creating an ensemble of $f_m$'s.
The simulated selectors $f_m$ are controlled by two parameters: $\nmeaningful$ and $\probability$. 
The first parameter, $\nmeaningful$, can be interpreted as the number of the features that are useful for the task.
The second parameter, $\probability$, can be considered as the probability value that reflects the uncertainty derived from both feature selectors and the dataset.
We estimate these parameters by using the results obtained from running the real selector.
For now, we assume that these parameters have already been estimated.
The estimation method is described in Section~\ref{sec:model-parameter-estimation}.

\subsubsection{Simulated Feature Selector Modeling}\label{sec:simulated-feature-selector-modeling}
In this problem setting, weak selectors use the same base feature selection algorithm,
and their behavior varies according to the algorithmic randomness, such as the stochastic training and dataset subsampling.
To mimic these behaviors, we model the simulated weak selector $f_m$ as follows:
Each feature selector $f_m$ has a unique feature subset $S_m \subset S$ ($|S_m|=\ntarget$) from which $f_m$ tends to choose more frequently than from the other features.
The definition of $S_m$ is in Section~\ref{sec:s-m-configuration}.
To model the effect of the randomness of the feature selectors and sub-datasets, we determine the behavior of $f_m$ as follows (Algorithm~\ref{alg:simulated-feature-selector}):
$f_m$ samples features from $S_m$ with probability $\probability \in [0, 1]$ and from $S\setminus S_m$ with probability $1-\probability$ without replacement.
Algorithm~\ref{alg:simulated-feature-selector} repeats the selection until $f_m$ picks up all the features, then ranks the features by the order they are selected.

We assume that the parameter $\probability$ is the same for all the weak selectors $f_m$ and is determined by the base selector $f$ and dataset $D$.
If $\probability=1$, the top $\ntarget$ features of the rank determined by $f_m$ are $S_m$ since we have $|S_m|=\ntarget$.
Therefore, $S_m$ can be considered as the feature set selected by $f_m$ under the ideal situation with no noise in the dataset.
However, the reality is $\probability < 1$, implying that $f_m$ chooses other features with a relatively low probability.

\subsubsection{Configuration of Feature Set \texorpdfstring{$S_m$}{}}\label{sec:s-m-configuration}

It is natural to assume that the feature subsets selected by the weak selectors are likely to have large overlaps.
To model this, we set the feature set $S_m$ as follows.
Let $S'=\{1, \ldots, \nmeaningful\}\subset S$ be a pool of features from which the base feature selectors are likely to select.
For each $m$, we uniformly randomly sample $\ntarget$ features from $S'$
and form the feature subset $S_m$.
Note that we defined $S'$ as the first $\nmeaningful$ indices of $S$ to simplify the algorithm.
Any subset of $S$ with $\nmeaningful$ elements is eligible for $S'$.

\subsection{Model Parameter Estimation}\label{sec:model-parameter-estimation}

In this section, we propose a method to estimate that two parameters: $\nmeaningful$ and $\probability$.

\subsubsection{Estimation of Parameter \texorpdfstring{$\nmeaningful$}{}}\label{sec:n-meaningful-estimation-method}

We define the \textit{uniform feature selector} by a feature selector that chooses a feature subset of size $\ntarget$ uniformly randomly from $S$.
Recall that the parameter $\nmeaningful$ represents the number of features that are useful for the task. 
The key to estimating $\nmeaningful$ is that, when $\probability$ is large, any feature in $S'$ is chosen by the simulated feature selector more likely than by the uniform feature selector.
Mathematically, the following theorem claims that the first iteration of Algorithm~\ref{alg:simulated-feature-selector} selects a feature from $S'$ more likely than from $S\setminus S'$.
\begin{theorem}\label{thm:justification-of-n-meaningful-estimation}
    Suppose we choose one feature $s$ from $S$ by the following process.
    First, we choose $\ntarget$ features from $S'$ uniformly randomly (we denote the set of selected features by $S_0$).
    Then, we choose a feature $s$ uniformly randomly from $S_0$ with probability $\probability$ and from $S\setminus S_0$ with probability $1-\probability$.
    If $\probability >\ntarget / \nfeatures$, then, for any feature in $S'$, the probability that the feature is chosen as $s$ is higher than $1/\nfeatures$.
    If $\probability < \ntarget / \nfeatures$, the opposite is true.
\end{theorem}
We give the proof of Theorem~\ref{thm:justification-of-n-meaningful-estimation} in Section~\ref{sec:proof-of-theorem1} in the supplementary.
On the basis of this consideration, we propose to estimate $\nmeaningful$ by the following procedure.
We run the uniform feature selector the same number of times as the number of weak selectors (i.e., $\nensemble$ times).
Let $\threshold$ be the number of times the most frequent feature is selected.
Then, we run the base feature selector $\nensemble$ times.
Let $c_i$ be the number of times the $i$-th feature is chosen for the feature selectors.
We define $\nmeaningful$ by the number of $c_i$ that is larger than $\threshold$, that is,
\begin{align*}
    \threshold &= \max \{c'_i \mid i = 1, \ldots, \nfeatures\}, \\
    \nmeaningful &= \sum_{i=1}^{\nfeatures}1[c_i > \threshold].
\end{align*}
Here, $c'_i$ is the number of uniform feature selectors that selected the $i$-th feature, and $1[\cdot]$ is an indicator function, i.e., $1[P]=1$ if the proposition $P$ is true and $0$ otherwise.

\subsubsection{Estimation of Parameter \texorpdfstring{$\probability$}{}}\label{sec:p-estimation-method}

The parameter $\probability$ is estimated from the stability of the base selector and $\nmeaningful$ that was determined in Section~\ref{sec:n-meaningful-estimation-method}.
As explained in Section~\ref{sec:stability-estimation-of-ensemble-efature-selectors}, we observe a monotonic relationship between $\probability$ and the stability value: the smaller $\probability$ is, the lower the stability is.
This is intuitively explained as follows.
On the one hand, as $\probability$ becomes large, feature selectors select features more often from the feature set $S'$.
On the other hand, in a high-dimensional setting, we usually expect that the number of useful features is much smaller than the feature dimension, that is, $\nmeaningful \ll \nfeatures$.
For instance, it is estimated that $\nmeaningful (=|S'|)$ is approximately $60$ and is much smaller than $\nfeatures (=|S|) =2000$ in Colon dataset used in the experiments of Section~\ref{sec:results}.
Under the condition $\nmeaningful \ll \nfeatures$, the more likely feature selectors select from $|S'|$, the larger the overlap between the outputs of the feature selectors will be, implying that the stability increases.

Using this monotonic relationship between $\probability$ and the stability value, we propose to estimate the parameter $\probability$ as follows: we apply a series of simulated ensembles (Algorithm~\ref{alg:simulated-feature-selector}) for discrete values of $\probability$, and choose the parameter $\probability$ whose estimated stability is the closest to that of the real selector.
In this study, we used the nine values of $\probability=0.1, 0.2, \ldots, 0.9$.

\subsubsection{Verification of Parameter \texorpdfstring{$\nmeaningful$}{}}\label{sec:n-meaningful-verification}

In this section, we propose another method that estimates $\probability$ and simultaneously verifies the estimated $\nmeaningful$ at the cost of additional simulations.
We need this method because the estimation method in Section~\ref{sec:n-meaningful-estimation-method} may overestimate $\nmeaningful$ when $\probability$ is small, as suggested by Theorem~\ref{thm:justification-of-n-meaningful-estimation}.
We can intuitively understand it as follows:
When $\probability$ is small, the algorithm generates many false negatives (i.e., features that are in $S'$ but are not selected only less frequently compared with the threshold $\threshold$), and false positives (i.e., features that are not in $S'$ but are selected more frequently than the threshold).
Under the condition $|S'| \ll |S|$, it is expected that false negatives occurs more likely than false positives, which results in the overestimation of $\nmeaningful$.
We verify this analysis empirically in Section~\ref{sec:p-estimation}.

Given the parameters $\probability$ and $\nmeaningful$, we can \textit{recompute} $\nmeaningful$ by comparing the uniform feature selector and the simulated selector constructed from these parameters.
More specifically, we use the same algorithm to estimate $\nmeaningful$ as the one described in Section~\ref{sec:n-meaningful-estimation-method}, except that we use the simulated selector instead of the real selector.
Let $\parameterizedsimulator$ be the simulated selector constructed from the parameters $\nmeaningful$ and $\probability$, and $\mathcal{A}[f']$ be the algorithm for computing the $\nmeaningful$ value using the feature selector $f'$ (Section~\ref{sec:n-meaningful-estimation-method}).
If the simulated selector perfectly models the feature selector, it should satisfy the following fixed point equation:
\begin{equation}\label{eq:fixed-point-equation-for-n-meaningful}
    \nmeaningful = \mathcal{A}[\parameterizedsimulator].
\end{equation}
Therefore, by running the estimation algorithm $\mathcal{A}$ repeatedly, we can obtain the calibrated estimation of $\nmeaningful$ under certain conditions such that the iteration converges.

In the experiments below, we only run the algorithm once and confirm that Equation (\ref{eq:fixed-point-equation-for-n-meaningful}) holds, i.e., the initial estimation $\nmeaningful$ and the re-estimated value $\nmeaningfulverification=\mathcal{A}[\parameterizedsimulator]$ are identical (here, $\mathrm{v}$ in the notation means the \textit{verification} of $\nmeaningful$).

\subsection{Computational Complexity}\label{sec:computational-complexity}

In the total workflow of the proposed algorithm, the key measure of computational complexity is how many times we run the real selectors.
The algorithm executes the real selector $O(\nstability + \nensemble)$ times.
If we compute the stability of the ensemble selectors naively, we need to run the feature selector $O(\nstability \times \nensemble)$ times.
See Section~\ref{sec:computational-complexity-detail} in the supplementary for the derivation.

\section{Results}\label{sec:results} % Section 3
\begin{table}[t]
\centering
\caption{Dataset specification. \label{tab:dataset-specification}}
\resizebox{\columnwidth}{!}{
{\begin{tabular}{@{}lllll@{}}
\toprule
    Name & Feature Type & Label Type & Dim. & Sample Size \\
\midrule
    Colon & Discrete & Binary & 2000 & 62 \\
    Lymphoma & Discrete & Multi-class (9) & 4026 & 96 \\
    Prostate & Continuous & Binary & 5966 & 102 \\
\bottomrule
\end{tabular}}{}
}
\end{table}

\subsection{Experiment Settings}\label{sec:evaluation}\label{sec:implementation} 

To demonstrate the applicability of the proposed method with various data types, we conducted the experiments using three datasets of microarray gene expression data: Colon~\citep{ding2005minimum}, Lymphoma~\citep{ding2005minimum}, and Prostate~\citep{NIPS2010_09c6c378} datasets.
Table~\ref{tab:dataset-specification} shows the specifications of the datasets we used for the experiments.
These datasets differ in the type of feature vectors (discrete/continuous) and labels (binary/multi-class).
All datasets are included in the scikit-feature package~\citep{li2018feature}.

We used a trained random forest as a base feature selector.
The trained model gives an importance score to each feature.
We used another random forest as a predictor for evaluating the performance of the selected features.
We employed the mean of the rank in aggregating the results of the weak selectors as an ensemble algorithm as this is a common choice~\citep{10.1007/978-3-540-87481-2_21,he2010stable}.

We employed the pair-wise Jaccard similarity as the stability index.
Since all the datasets in this study are for classification problems, we used the accuracy as a $K$-classification task for the prediction performance when the dataset has $K$ label types.

See Sections~\ref{sec:evaluation-procedure-details}--\ref{sec:implementation-details} for the evaluation details.

In the main article, we mostly explain the results for Colon dataset.
We refer to Section~\ref{sec:additional-experiment-results} in the supplementary for the results of the other datasets (Lymphoma, Prostate).
We have additional discussions not included in the main article in Sections~\ref{supsec:discussion} and~\ref{supsec:future_direction} of the supplementary.

\subsection{Parameter Setting of \texorpdfstring{$\ntarget$}{}}

We set $\ntarget = 20, 40$, and $60$ for Colon, Lymphoma, and Prostate datasets, respectively,
and ran the ensemble selector consisting of random forest feature selectors.
Figure~\ref{fig:prediction_accuracy,n-redundant-estimation} (left) shows the prediction accuracy scores for Colon dataset, with various parameter settings.
Since, the classifier achieved high accuracy regardless of the number of trees in the random forest feature selector (referred to as $\ntrees$ hereafter), we use these $\ntarget$ values in the following experiments.

\subsection{Estimation of Parameter \texorpdfstring{$\nmeaningful$}{}}\label{sec:n-meaningful-estimation-result}

\subsubsection{Determination of Threshold \texorpdfstring{$\threshold$}{}}\label{sec:determine-t-random-max}

We first determine the threshold $\threshold$, which is needed to estimate $\nmeaningful$.
Figure~\ref{fig:prediction_accuracy,n-redundant-estimation} (right) compares the distribution of the frequency counts of the real and uniform selectors in a single estimation of $\threshold$ for Colon dataset (see Figure~\ref{supfig:n-redundant-estimation} for Lymphoma and Prostate datasets).
We set $\ntrees = 500$ for the random forest used as a feature selector.
When we used the uniform selector, the distribution was concentrated at low frequencies.
On the contrary, the distribution was more heavy-tailed when we used the real selector.
These distributions differed significantly as expected.
This implies that we can find the features that the real selectors prefer to choose.

\begin{figure}[!tpb]
    \centering
    \includegraphics[width=0.47\linewidth]{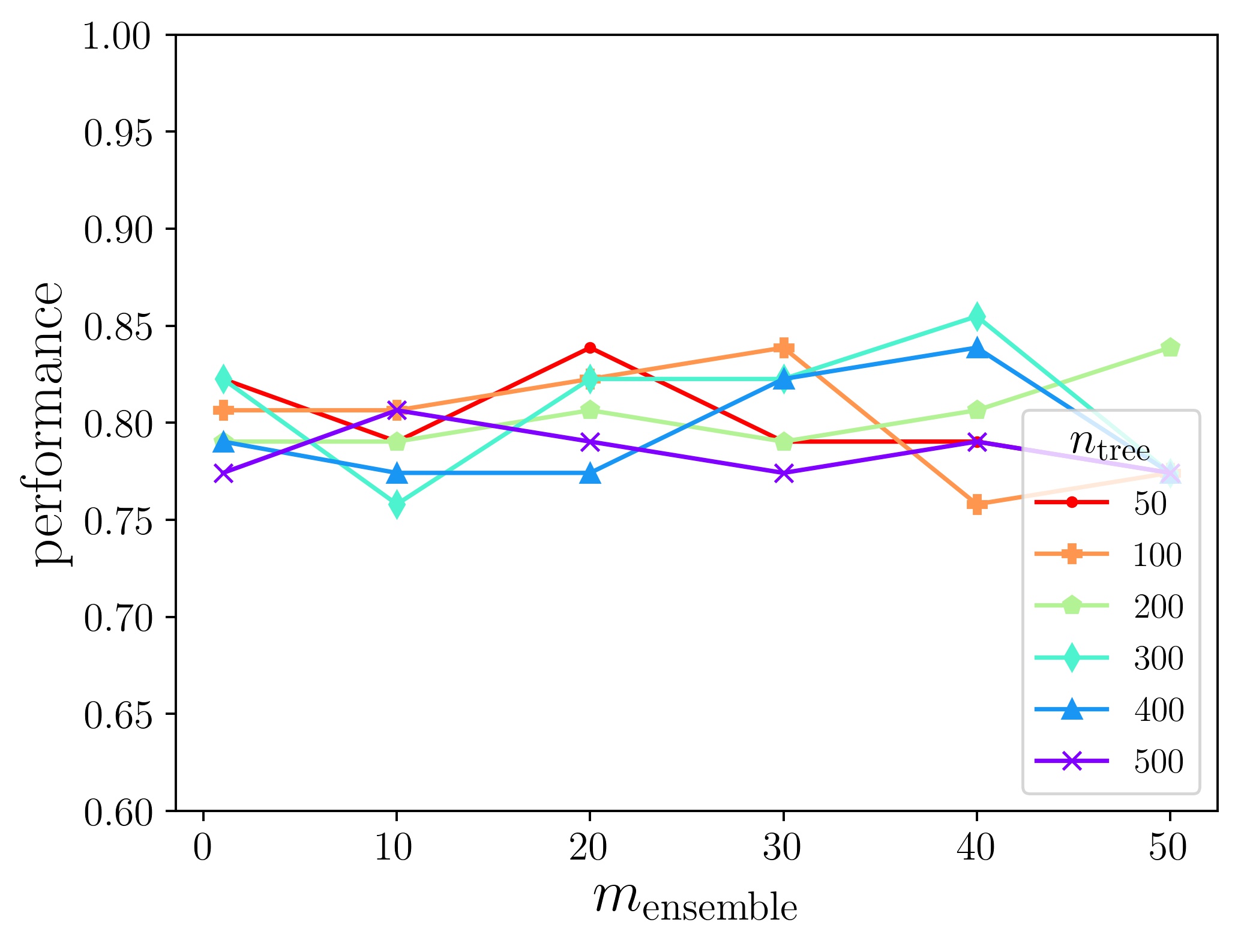}
    \includegraphics[width=0.47\linewidth]{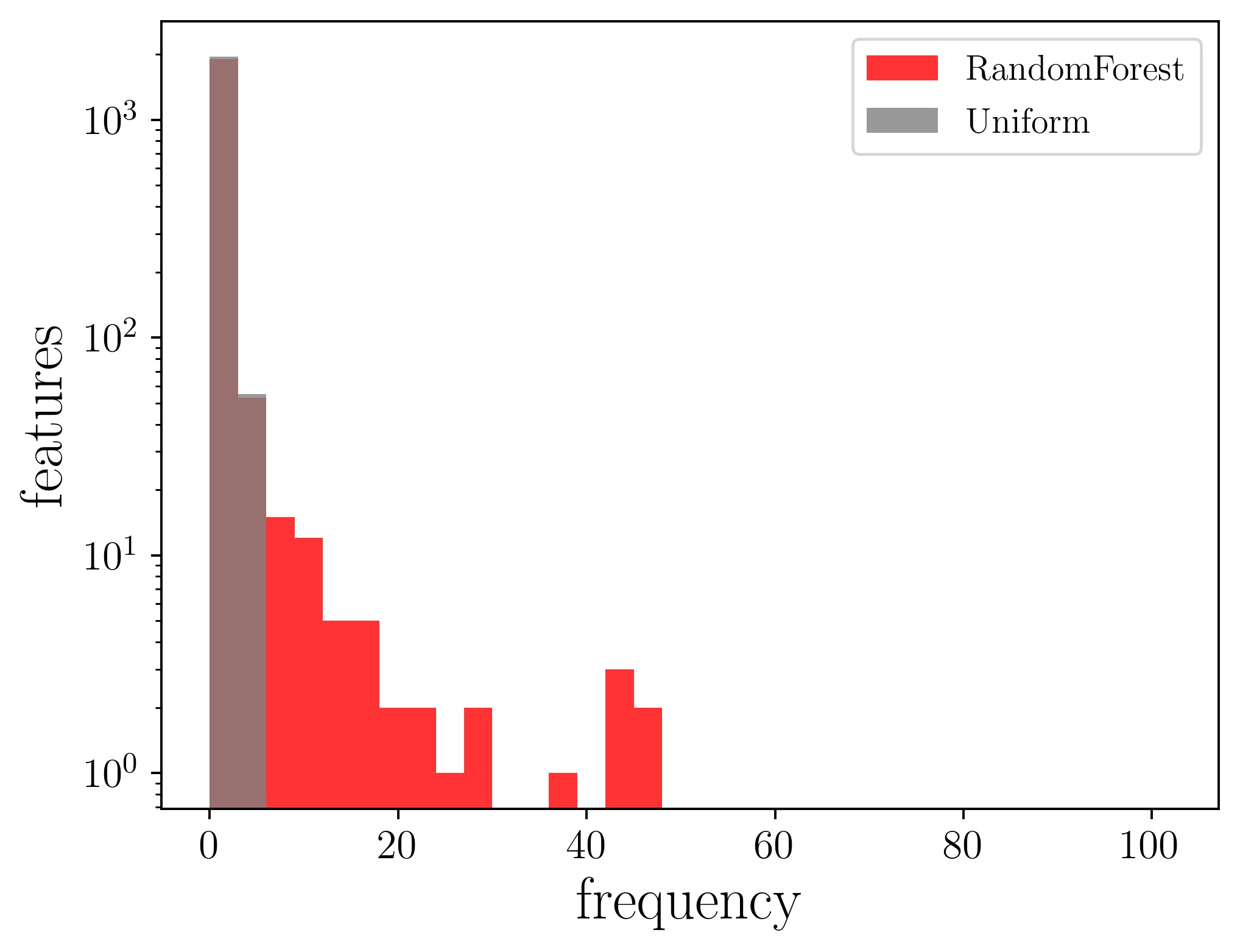}
    \caption{(Left) Prediction accuracy for Colon dataset. (Right) Histogram of features' frequencies selected by the random forest and the uniform feature selectors for Colon dataset.
    }\label{fig:prediction_accuracy,n-redundant-estimation}
\end{figure}

Recall that the threshold $\threshold$ is the maximum frequencies chosen by uniform selectors.
When we computed the threshold $1000$ times, the mean and standard deviation of the threshold was $\threshold = 4.640 \pm 0.636$ for Colon dataset. 
Since we estimated $\nmeaningful$ by the number of features that were selected more frequently than $\threshold$, the stability of $\threshold$ could affect the robust estimation of $\nmeaningful$.
As the standard deviation was relatively small,
we see that this method determined $\threshold$ in a stable manner.

\subsubsection{Effect of Feature Selector Models on Estimation of Parameter \texorpdfstring{$\nmeaningful$}{}}\label{sec:effect-of-base-models-on-parameter-estimation}

Table~\ref{tab:n-redundant-vs-n-trees} shows the estimated value of $\nmeaningful$ for various numbers of trees $\ntrees$ in a random forest feature selector.
We estimated the value of $\nmeaningful$ $1000$ times using the same values of $\threshold$ as those in Section~\ref{sec:determine-t-random-max} and computed its mean and standard deviation.

We observed that $\nmeaningful$ depended not only on the dataset but also on the complexity of the feature selector.
More specifically, in Prostate dataset, the estimated $\nmeaningful$ monotonically decreased as $\ntrees$ increased.
However, no trend of monotonic changes was observed in $\nmeaningful$ with respect to $\ntrees$ for Colon and Lymphoma datasets.
We discuss the relation of $\ntrees$ and $\nmeaningful$ in more detail in Section~\ref{supsec:effect-of-base-models-on-parameter-estimation} in the supplementary.

Based on this result, we set $\nmeaningful = 60, 150,$ and $200$ for Colon, Lymphoma, and Prostate datasets, respectively in the following experiments.
They correspond to the random forest feature selector with $300$ trees for Colon and Lymphoma datasets and with $100$ -- $200$ trees for Prostate dataset.

\begin{table}[!t]
\centering
\caption{Estimation of the parameter $\nmeaningful$.\label{tab:n-redundant-vs-n-trees}} {
\begin{tabular}{@{}lllll@{}}
\toprule
    $\ntrees$ & Colon & Lymphoma & Prostate \\
\midrule
    50 & 45.6 $\pm$ 7.1 & 140.0 $\pm$ 20.3 & 210.3 $\pm$ 16.8 & \\
    100 & 53.0 $\pm$ 8.2 & 140.4 $\pm$ 14.8 & 210.8 $\pm$ 15.9  & \\
    200 & 56.8 $\pm$ 6.0 & 158.3 $\pm$ 15.4 & 190.5 $\pm$ 12.1  & \\
    300 & 60.1 $\pm$ 6.1 & 153.2 $\pm$ 14.9 & 186.9 $\pm$ 12.7  & \\
    400 & 56.2 $\pm$ 7.2 & 155.4 $\pm$ 13.9 & 179.6 $\pm$ 10.2  & \\
    500 & 59.1 $\pm$ 6.6 & 154.5 $\pm$ 10.2 & 179.0 $\pm$ 10.8  & \\
\bottomrule
\end{tabular}}{}
\end{table}

\subsection{Estimation of Parameter \texorpdfstring{$\probability$}{} and Verification of Parameter \texorpdfstring{$\nmeaningful$}{}}\label{sec:p-estimation}

Figure~\ref{fig:n-redundant-reestimation,p-estimation} (left) shows the relationship between the initial estimate of $\nmeaningful$ and its corresponding re-estimated value $\nmeaningfulverification$.
These two quantities agreed when they are small.
However, $\nmeaningfulverification$ decreased when $\nmeaningful$ was large.
In addition, as the parameter $\probability$ became smaller, $\nmeaningfulverification$ started to deviate from $\nmeaningful$ at a smaller value.
This fact agrees with the qualitative discussion in Section~\ref{sec:n-meaningful-verification}.

Figure~\ref{fig:n-redundant-reestimation,p-estimation} (right) plots the parameter pairs $\nmeaningful$ and $\probability$ that are compatible with the real selectors in terms of stability for Colon dataset.
More concretely, we first computed the stability values for the single real feature selector, which were $0.1$.
Next, we computed the stability of the simulated selector for various combinations of $\nmeaningful$ and $\probability$.
Then, we selected the parameters whose stability value was close to that of the real selector (blue triangles in Figure~\ref{fig:n-redundant-reestimation,p-estimation} (right)).
Similarly, we computed the $\nmeaningfulverification$ values and selected the parameters whose $\nmeaningfulverification$ was close to $\nmeaningful$ (orange circles in Figure~\ref{fig:n-redundant-reestimation,p-estimation} (right)).
The intersections of the blue triangles and the orange circles are the desired parameters in terms of the stability and estimation of $\nmeaningful$.

We estimated $\probability=0.7$ for Colon dataset from Figure~\ref{fig:n-redundant-reestimation,p-estimation} (right).
In addition, we observe that $\nmeaningful$ and $\nmeaningfulverification$ were both $60$, implying that the parameters $\probability$ and $\nmeaningful$ were compatible in the sense that they satisfied Equation~(\ref{eq:fixed-point-equation-for-n-meaningful}).
We discuss results for the other datasets (Lymphoma, Prostate) in detail in Section~\ref{sec:additional-experiment-results} in the supplementary.

\begin{figure}[t]
    \centering
    \includegraphics[width=0.60\linewidth]{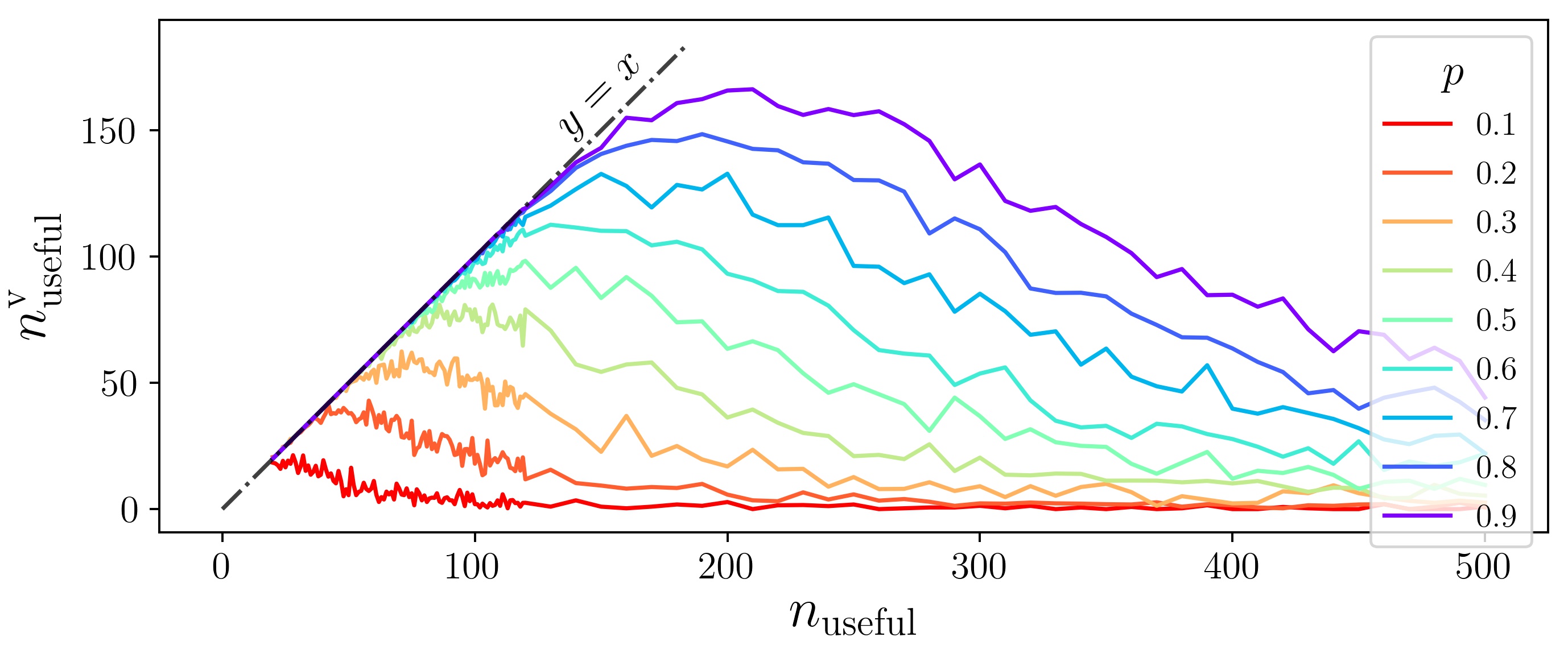}
    \includegraphics[width=0.35\linewidth]{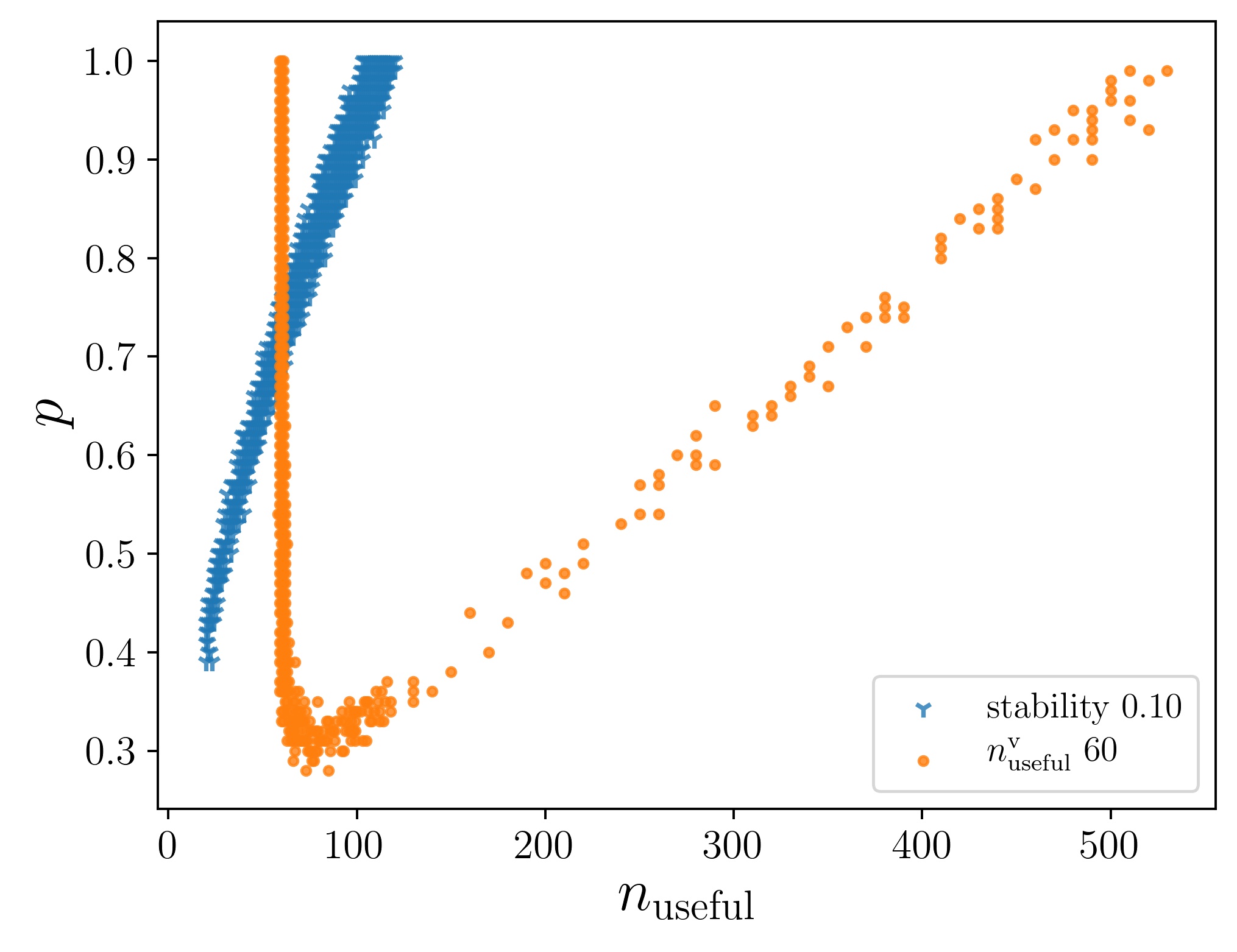}
    \caption{
    (Left) The verification value $\nmeaningfulverification$ of $\nmeaningful$ computed from the simulated selector constructed from the parameters $\probability$ and $\nmeaningful$.
    If the parameter $\nmeaningful$ satisfies $\nmeaningful = \nmeaningfulverification$ (the dashed line), it satisfies the fixed point equation~(\ref{eq:fixed-point-equation-for-n-meaningful}).
    (Right) Estimation of the parameter $\probability$. We used the configuration of Colon dataset. The blue triangles are parameter pairs $(\nmeaningful, \probability)$ such that the estimated stability is within $0.1\pm 0.01$. The orange circle points are parameter pairs whose $\nmeaningfulverification$ is within $60\pm 2$. $\nmeaningful$ ranges over $\nmeaningful=20,\ldots, 810$. $\probability$ ranges over $\probability=0.01, \ldots, 1.00$.
    }\label{fig:n-redundant-reestimation,p-estimation}
\end{figure}

\subsection{Stability Estimation of Ensemble Feature Selectors}\label{sec:stability-estimation-of-ensemble-efature-selectors}

Figure~\ref{fig:stability-estimation,stability-actual} (left) shows the stability of the simulated selectors constructed from $\nmeaningful$, which were estimated in Section~\ref{sec:n-meaningful-estimation-result} and verified in Section~\ref{sec:p-estimation}.
Although $\probability$ was already determined as a single value $0.7$, we plotted the estimated stability values for various different $\probability$ for later use.
Figure~\ref{fig:stability-estimation,stability-actual} (right) appears similar to Figure~\ref{fig:stability-estimation,stability-actual} (left), but was obtained from the real selector using various $\ntrees$.

We compared the graph of $\probability=0.7$ in Figure~\ref{fig:stability-estimation,stability-actual} (left), estimated in Section~\ref{sec:p-estimation}, and the graph of $\ntrees=300$ in Figure~\ref{fig:stability-estimation,stability-actual} (right), corresponding $\nmeaningful=60$ as estimated in Section~\ref{sec:effect-of-base-models-on-parameter-estimation}.
We observed that they behaved similarly, implying that the simulated selector with $\probability=0.7$ simulated effectively the stability of the real selectors with $\ntrees=300$.

Next, we analyzed how the simulation result changed as we changed $\probability$.
Figure~\ref{fig:stability-estimation,stability-actual} (left) verifies the monotonic relationship between $p$ and the stability value that we intuitively explain in Section~\ref{sec:p-estimation-method}.

Finally, we pay attention to the stability values when we ensemble many weak feature selectors.
When the number of weak selectors $\nensemble$ was greater than or equal to $30$, the estimated stability was stable at approximately $0.2$ regardless of $\probability$ (Figure~\ref{fig:stability-estimation,stability-actual}, left).
This stability value is close to that in Figure~\ref{fig:stability-estimation,stability-actual} (right) when $\ntrees$ was greater than or equal to $200$ and $\nensemble$ was greater than $40$.
On the contrary, differently from the simulation result, the actual stability values when $\ntrees$ is smaller than $200$ were lower than those when $\ntrees$ is larger than or equals to $200$ (Figure~\ref{fig:stability-estimation,stability-actual}, right).
We hypothesize that this inconsistency occurred because the estimated $\nmeaningful$ values were different between different $\ntrees$, as we discuss in Section~\ref{supsec:discussion} in the supplementary.

\subsection{Computational Time}\label{sec:computational-time}

We compared the computational time of the proposed stability estimation method with the stability calculation using the real selectors.
For a fair comparison, when we measured the time of the real selectors, we neither trained the random forest classifiers nor evaluated their prediction performance, as these were not performed for the simulated selectors either.
Note that we measured the time required to run the simulated weak selectors for the given parameters $\nmeaningful$ and $\probability$.
Therefore, we needed additional time for estimating the parameters, which is theoretically faster than the naive computation of the real selector (Section~\ref{sec:computational-complexity}).
A single CPU was used for the computation.
In particular, weak selectors were performed sequentially.

Figure~\ref{fig:computation-time} shows the elapsed time to calculate the stability for the real feature selectors and the stability estimation by simulation, respectively.
We used the parameters for Colon dataset.
See Figures~\ref{supfig:computation-time-real} and~\ref{supfig:computation-time-simulation} in the supplementary for the results for Lymphoma and Prostate datasets.
Note that the scales of the vertical axes are different between the results for the real and the simulated ensemble feature selectors.
This shows that our proposed method is much faster than the actual stability computation.
For the real selectors, the computation time was almost proportional to the number of weak selectors (Figure~\ref{fig:computation-time}, left).
By contrast, we observed a moderate increase in time for the case of simulated selectors when we add weak selectors (Figure~\ref{fig:computation-time}, right).

\begin{figure}[t]
    \centering
    \includegraphics[width=0.47\linewidth]{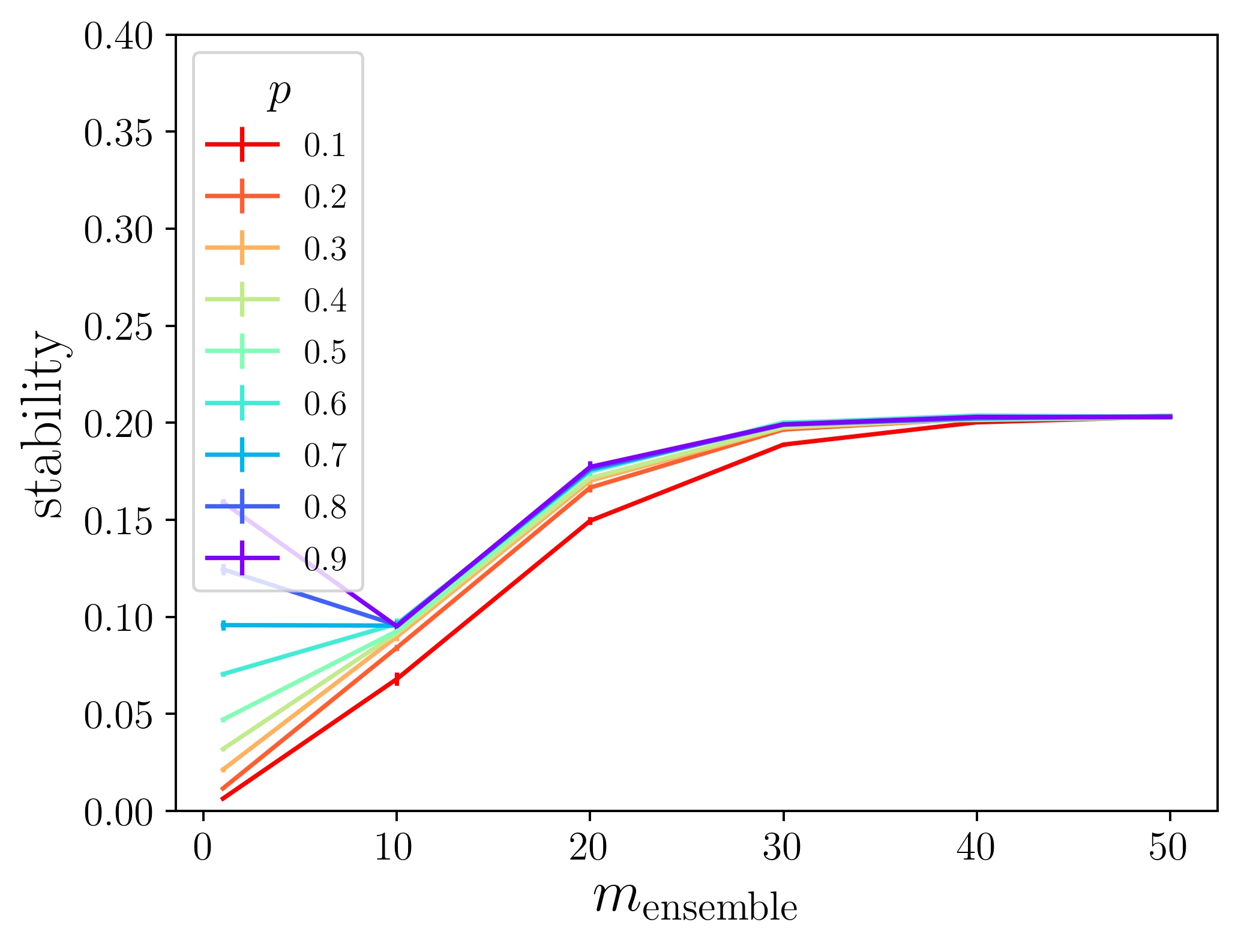}
    \includegraphics[width=0.47\linewidth]{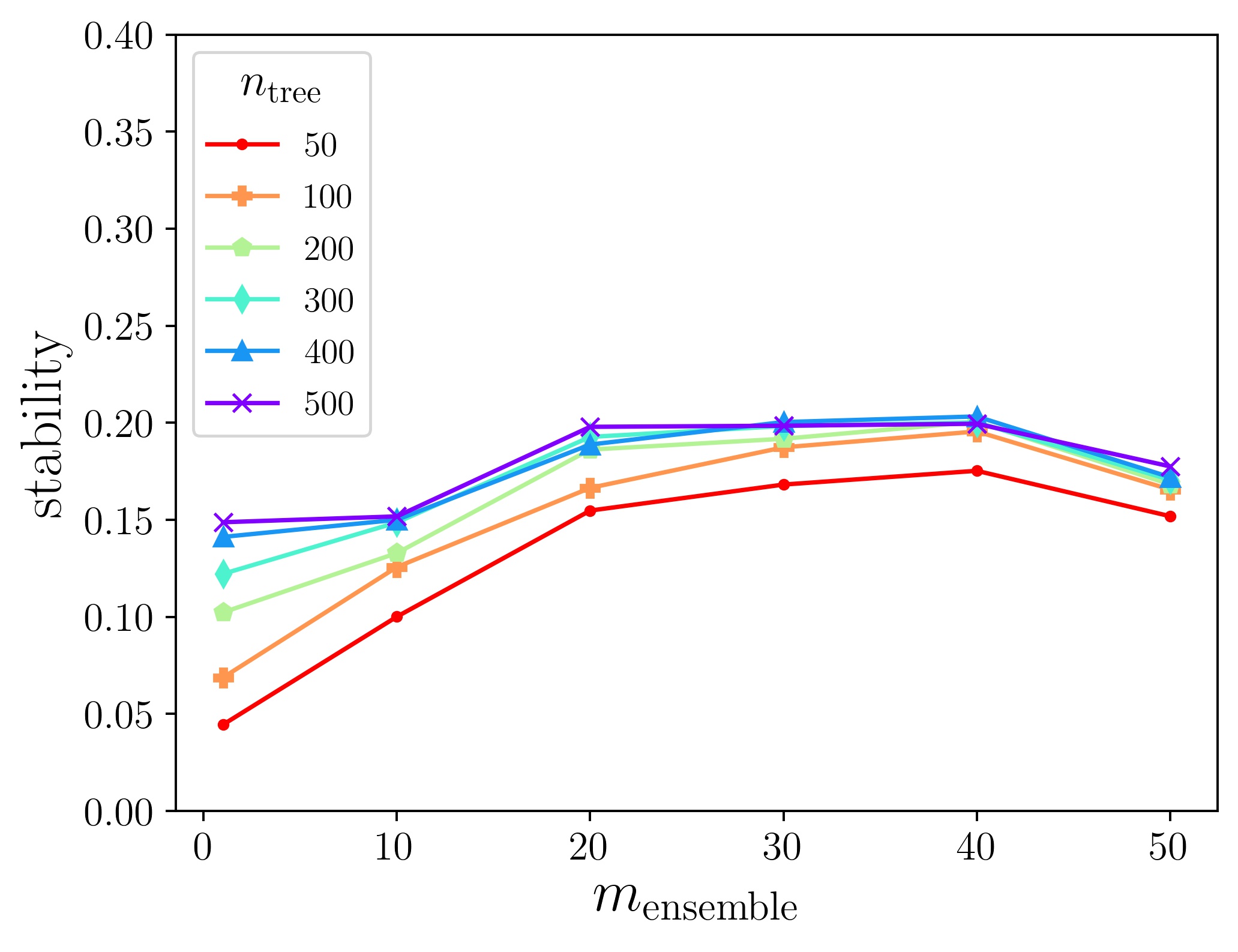}
    \caption{Estimation of stability of the simulated (left) and real ensemble selectors (right) for Colon dataset. 
    }\label{fig:stability-estimation,stability-actual}
\end{figure}

\begin{figure}[!tpb]
    \centering
    \includegraphics[width=0.47\linewidth]{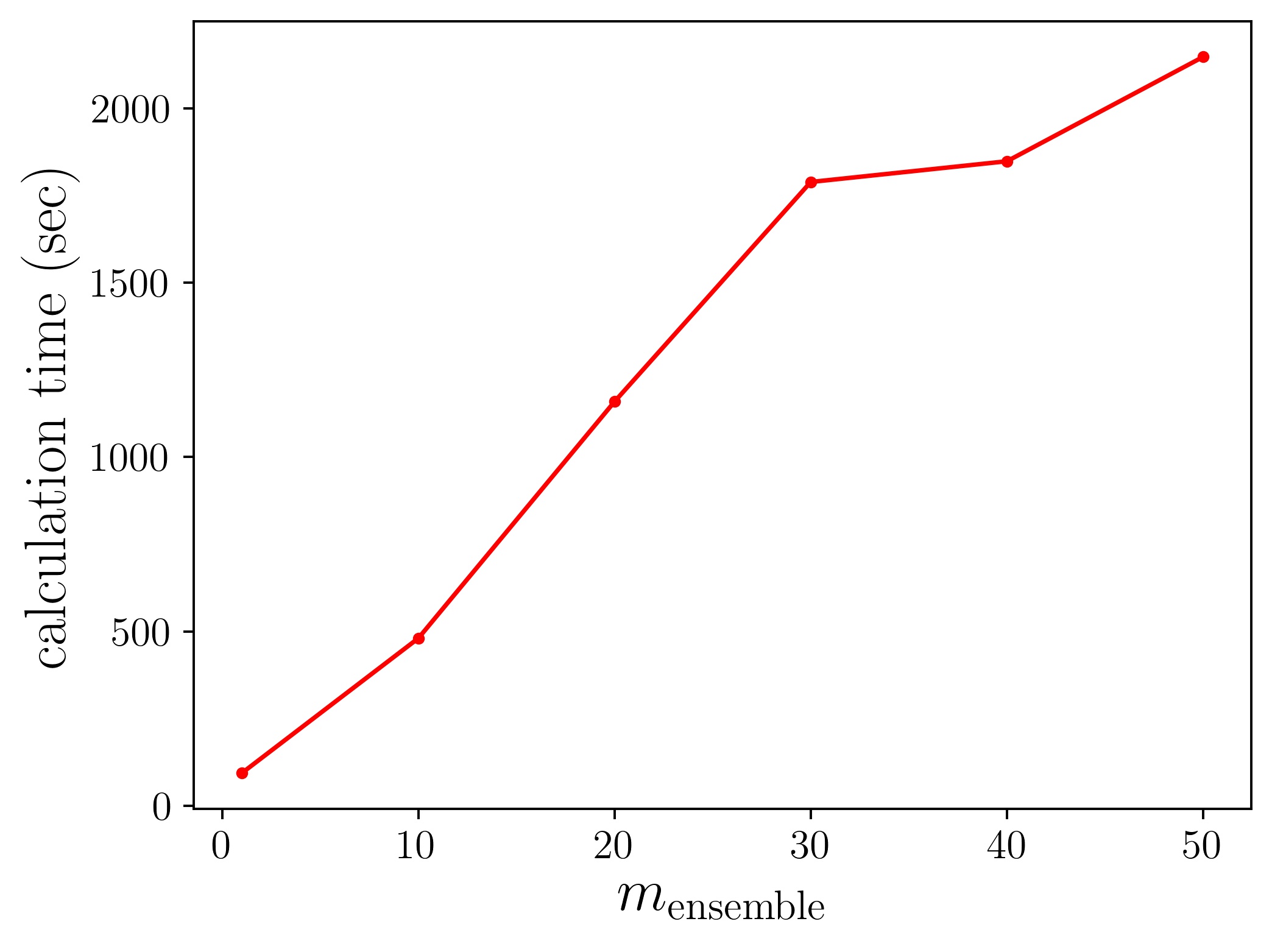}
    \includegraphics[width=0.47\linewidth]{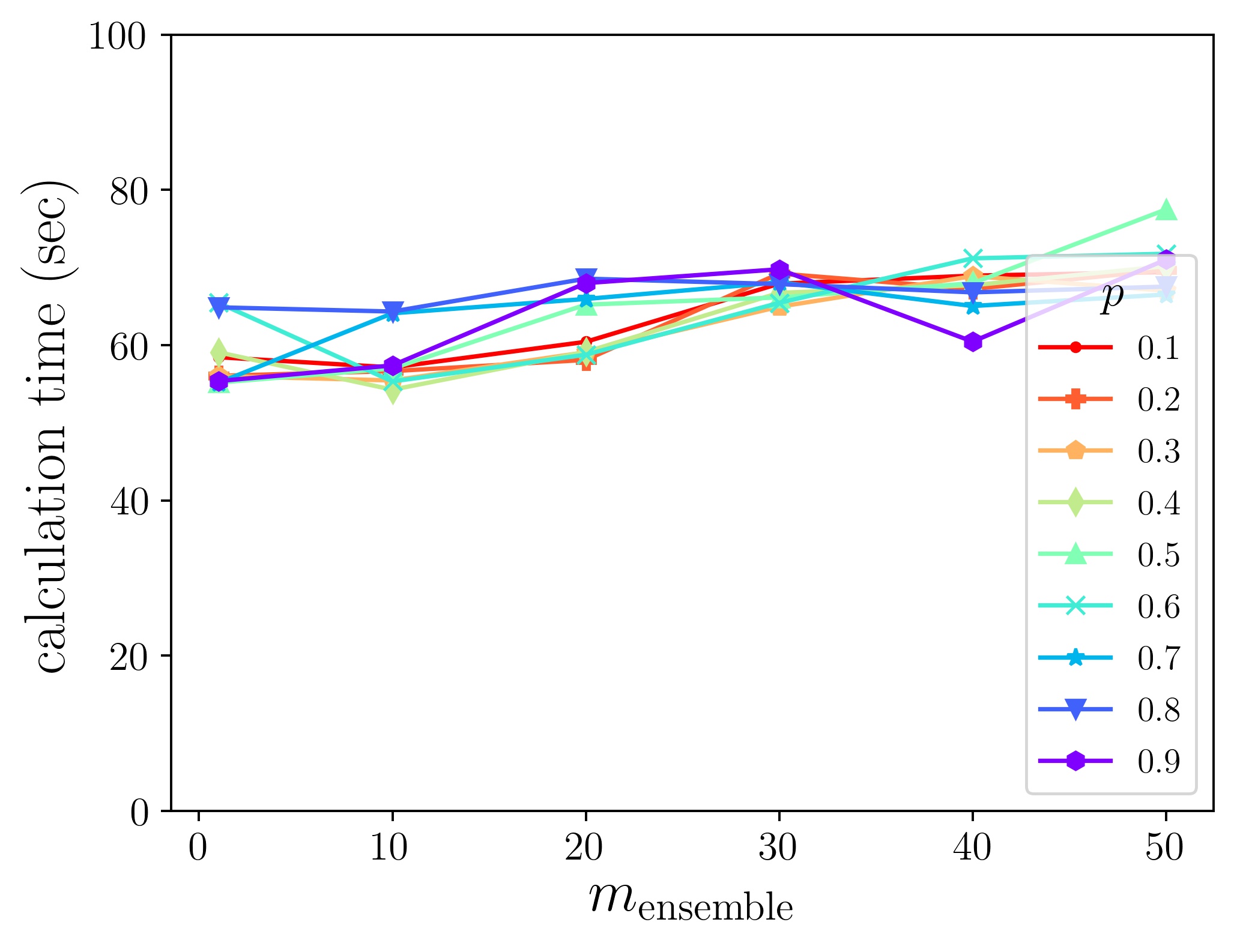}
    \caption{Computation time for Colon dataset using the real (left) and simulated (right) ensemble selectors. We used  $\ntrees=500$ as a real selector.
    }\label{fig:computation-time}
\end{figure}

\section{Conclusion}

In this paper, we proposed a fast estimation method for the stability of ensemble feature selectors.
The idea is to construct simulators which mimic weak selectors using two interpretable parameters and to ensemble simulated feature selectors instead of real ones.
Theoretically, the proposed method reduces the number of executions of the real feature selectors.
Using three cancer datasets, we demonstrated that the proposed method can accurately estimate the stability of the ensemble feature selectors with a small amount of computation.
Our method helps judge whether an ensemble feature selection algorithm is effective in terms of stability without actually run it many times.
Thus, it extends the applicability of ensemble feature selection.

\section*{Acknowledgments}
We would like to thank Editage (\url{www.editage.com}) for English language editing.

\bibliographystyle{natbib}
\bibliography{main}

\newpage
\appendix

\section{Notation Table}

Table~\ref{tab:notation-table} explains the notations we adopted throughout the paper.
Here, $|A|$ denotes the cardinality of the set $A$ (i.e., the number of elements in $A$).

\section{Additional Analyses}

\subsection{Proof of Theorem~\ref{thm:justification-of-n-meaningful-estimation}}\label{sec:proof-of-theorem1}

In this section, we give a proof for Theorem~\ref{thm:justification-of-n-meaningful-estimation}, which gives an intuitive justification of the estimation method for the parameter $\nmeaningful$ that we explained in Section~\ref{sec:n-meaningful-estimation-method}.
    \begin{proof}[Proof of Theorem~\ref{thm:justification-of-n-meaningful-estimation}]
        For notational simplicity, we denote $\nf = \nfeatures$, $\nm = \nmeaningful$, and $\nt = \ntarget$, respectively.
        Taking a feature $x$ in $S'$.
        We find that the probability that $x$ is in $S_0$ is $\frac{\nt}{\nm}$.
        Therefore, the probability $p_0$ that the feature $x$ is selected is
        \begin{align*}
        p_0 &=\frac{\nt}{\nm} \times p \times \frac{1}{\nt} + \left(1 - \frac{\nt}{\nm}\right) \times (1-p) \times \frac{1}{\nf - \nt}\\
            &=\frac{(\nf - \nm)p + (\nm - \nt)}{\nm(\nf - \nt)}.
        \end{align*}
        The probability that the uniform feature selector selects $x$ is $p_1 = \frac{1}{\nf}$.
        By direct calculation, we have
        \begin{align*}
            p_0 - p_1 &= \frac{\nf-\nm}{\nm \nf (\nf-\nt)}(\nf p - \nt).
        \end{align*}
        Therefore, $p_0 > p_1$ if and only if $p > \frac{\nt}{\nf}$.
    \end{proof}
Let us consider the situation where we run the same number of the base feature selectors and uniform feature selectors, and count the number of times each feature is top-ranked.
Suppose that the simulated feature selector can perfectly simulate the base feature selector and that the parameter $\probability$ is large.
Then, Theorem~\ref{thm:justification-of-n-meaningful-estimation} suggests that, if we run a sufficient number of feature selectors, the base feature selectors select features in $S'$ more often than the uniform feature selectors do with high probability.

\subsection{Computational Complexity}\label{sec:computational-complexity-detail}

We review the whole process and count the total number of executions.
The algorithm runs the real feature selectors $\nensemble$ times to estimate $\nmeaningful$ (Section~\ref{sec:n-meaningful-estimation-method}).
In the estimation of $\probability$, the algorithm compares the stability of the real and simulated feature selectors (Section~\ref{sec:p-estimation-method}).
To do so, the algorithm runs the real feature selector $\nstability$ times to compute the stability of the real feature selector, and run the simulated feature selector $\nstability \times \nprobability$ times to compute the stability values of the simulated feature selector.
Here, $\nprobability$ is the number of possible values that the parameter $\probability$ can take.
In our setting, since we used $\probability = 0.1, \ldots, 0.9$, we have $\nprobability =9$.
For the verification of $\nmeaningful$, we do not have to run the real feature selectors (Section~\ref{sec:n-meaningful-verification}).
Similarly, in the computation of the stability of the simulated ensemble feature selectors, the algorithm does not run the real feature selector and only runs the simulated feature selector $\nmeaningful \times \nstability$ times.

In summary, the algorithm executes the real feature selector $O(\nstability + \nensemble)$ times.
If we compute the stability of the ensemble feature selectors naively, we need to run the feature selector $O(\nstability \times \nensemble)$ times.
Therefore, we can significantly reduce the computational time.
Although the algorithm runs the simulated feature selectors $O(\nstability \times (\nprobability + \nensemble))$ times, because the simulation is faster than the actual feature selector, this overhead is comparably small, as can be seen in Section~\ref{sec:computational-time}.

\section{Additional Experiment Settings}

\subsection{Evaluation Procedure}\label{sec:evaluation-procedure-details}

Figure~\ref{supfig:evaluation-overview} is a schematic diagram of how to evaluate the stability and the prediction performance of a real feature selector.
We split a whole dataset into a training dataset for feature selection (\textit{training dataset 1}), a training dataset for a predictor (\textit{training dataset 2}), and a test dataset for evaluating the predictor (\textit{test dataset}).
Because the sample size of the datasets was small, we employed the leave-one-out cross-validation to evaluate the predictor in this study.
That is, the sample size of the test dataset was $1$.
We split the remaining data into two training datasets of approximately equal size.

We employed the pair-wise Jaccard similarity as the index of the stability. 
The pairwise Jaccard similarity $J$ of a feature selector is defined as follows:
\begin{equation*}
    J = \frac{2}{U(U-1)} \sum_{i=1}^{U} \sum_{j=i+1}^{U} \frac{|s_i \cap s_j|}{|s_i \cup s_j|}.
\end{equation*}
Here, $U=\nstability$ is the number of copies of a feature selector used for calculating its stability and $s_i \subset S$ is the feature selection result of the $i$-th copy of the feature selector.
Because all the datasets we dealt with in this study are classification problems, we used the accuracy as a $K$-classification task for the prediction performance when the dataset has $K$ label types.
If we use simulated feature selectors, we skip the training and evaluation of the predictor because they do not output the feature selection results.

\begin{figure}
    \centering
    \includegraphics[width=0.9\linewidth]{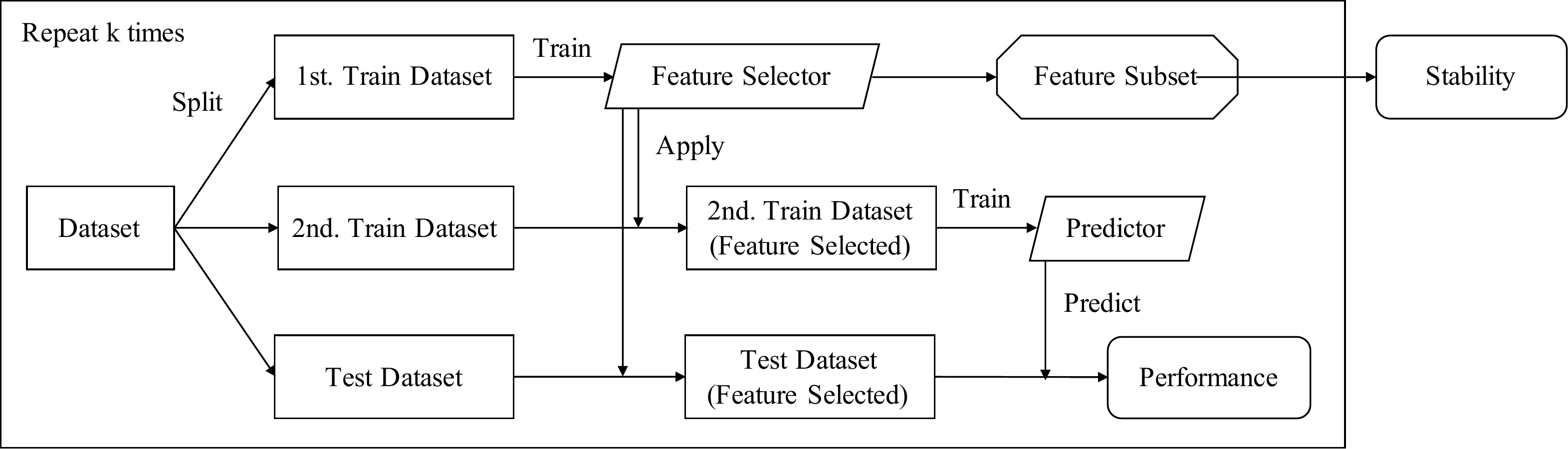}
    \caption{Evaluation overview.}\label{supfig:evaluation-overview}
\end{figure}

\subsection{Dataset}\label{sec:dataset-specification-details}

The followings are the details of the datasets we used in the experiments in the main article.

\subsubsection{Colon}

Colon dataset consists of discrete feature vectors and binary labels.
The sample consists of 22 tumor and 40 normal cells.
The feature vectors are 2000-dimensional, and their elements are discretized into three values $\{-2, 0, 2\}$.
The source of the dataset is~\citep{alon1999broad}.

\subsubsection{Lymphoma}

Lymphoma dataset consists of 96 data points, each with a discrete feature vector and a multi-class label.
The feature vector is 4026-dimensional.
Like those in Colon dataset, the feature vectors are discretized into three values.
Each label takes one of nine classes representing cancer sub-types.
The source of the dataset is~\citep{alizadeh2000distinct}.

\subsubsection{Prostate}

Prostate dataset consists of a continuous feature vector as continuous values and a binary label.
The sample has 102 data points consisting of 52 tumor and 50 normal cells, respectively.
Each feature is a 5966-dimensional continuous vector.
The source of the dataset is~\citep{singh2002gene}.

\subsection{Implementation Details}\label{sec:implementation-details}

We used Python 3 for all implementations of the models used in this study.
Feature selectors and machine learning models are based on the implementation of the scikit-learn library~\citep{scikit-learn}.
Specifically, we used the \texttt{RandomForestClassifier} module as an implementation of the random forest model.
We used the Gini index as a function for measuring the quality of a split in the algorithm.
We computed the importance score from the \texttt{feature\_importance\_} attribute of the \texttt{RandomForestClassifier} module.

\section{Additional Experiment Results} \label{sec:additional-experiment-results}

\subsection{Parameter Settings of \texorpdfstring{$\ntarget$}{}}

Figure~\ref{supfig:prediction_accuracy} shows the values of the prediction accuracy when the ensemble feature extractor is applied to Colon, Lymphoma, and Prostate datasets.
We chose $\ntarget = 20, 40, 60$ for Colon, Lymphoma, and Prostate datasets as the number of selected features $\ntarget$, respectively.
We confirmed that the classifier achieved high accuracy regardless of the number of trees $\ntrees$ in the random forest classifier.

\subsection{Estimation of the Parameter \texorpdfstring{$\nmeaningful$}{}}

Figure~\ref{supfig:n-redundant-estimation} compares the distribution of the frequency counts of the real and uniform feature selectors in a single estimation of $\nmeaningful$ for Colon, Lymphoma, and Prostate datasets.

Table~\ref{tab:threshold-by-random-feature-serlector} shows the mean and standard deviation of the threshold $\threshold$ computed $1000$ times for Colon, Lymphoma, and Prostate datasets.

\subsection{Estimation of the Parameter \texorpdfstring{$\probability$}{} and Verification of the Parameter \texorpdfstring{$\nmeaningful$}{}}

Figure~\ref{supfig:p-estimation} plots the parameter pairs $\probability$ and $\nmeaningful$ that are compatible with the real feature selectors in terms of stability for Colon, Lymphoma, and Prostate datasets.
Figure~\ref{supfig:stability-estimation} is the estimation of stability of the simulated ensemble feature selectors for these datasets.
Figure~\ref{supfig:stability-actual} is the actual stability value of the real ensemble feature selectors for these datasets.
Since we discussed the results of Colon dataset in the main article, we address the other two datasets in this section.

The stability value of the real feature selector, which is needed for the parameter estimation, was $0.1$ for Lymphoma and $0.2$ for Prostate datasets, respectively.

\subsubsection{Lymphoma Dataset}

We estimated from Figure~\ref{supfig:p-estimation} (middle) that $\probability=0.8$ and confirmed that the parameters $\probability$ and $\nmeaningful$ are compatible in the sense that $\nmeaningful$ satisfies the fixed point equation (\ref{eq:fixed-point-equation-for-n-meaningful}).
Figure~\ref{supfig:stability-estimation} (middle) shows the estimated stability of the simulated ensemble feature selection.
Figure~\ref{supfig:stability-actual} (middle) shows the values of the stability when the real ensemble feature selector was applied to Lymphoma dataset.

\subsubsection{Prostate Dataset}

We compared the simulated ensemble feature selectors in Figure~\ref{supfig:stability-estimation} (bottom) and the real ones in Figure~\ref{supfig:stability-actual} (bottom).
Similarly to the case of Colon dataset, both feature selectors increased the stability values as we increased the number of weak feature selectors, except for the simulation cases where $\nensemble=1$ and $\probability$ is greater than or equal to $0.7$.
However, the stability values differed when we ensembled a sufficient number of weak feature selectors: approximately $0.17$ for the simulation and $0.3$ for the real feature selector when $\ntrees=100$ and $200$ (corresponding to the estimated $\nmeaningful=200$, as estimated in Section~\ref{sec:effect-of-base-models-on-parameter-estimation}).

We hypothesize that this inconsistency in stability values was due to the failure in estimating $\nmeaningful$ in Section~\ref{sec:p-estimation}.
Figure~\ref{supfig:prostate-results-with-calibrated-n-redundant} is the stability estimation results for the Prostate dataset.
Differently from Figures~\ref{supfig:p-estimation} and~\ref{supfig:stability-estimation}, $\nmeaningful$ is changed from $200$ to $130$ in Figure~\ref{supfig:prostate-results-with-calibrated-n-redundant}.
We observe from Figure~\ref{supfig:prostate-results-with-calibrated-n-redundant} (top) that the equation $\nmeaningful=\nmeaningfulverification$ holds.
This implies that $\nmeaningful=130$ satisfies the fixed point equation (\ref{eq:fixed-point-equation-for-n-meaningful}).
Figure~\ref{supfig:prostate-results-with-calibrated-n-redundant} (bottom) is the stability estimation of the simulated ensemble feature selectors using the corrected value $\nmeaningful=130$.
The simulated stability value for a sufficiently large $\nensemble$ value was $0.3$ in this configuration.
This value is close to that for the real ensemble feature selectors we observed in Figure~\ref{supfig:stability-estimation} (bottom).
From this analysis, we conclude that the inconsistency can be fixed by appropriately finding $\nmeaningful$ such that Equation (\ref{eq:fixed-point-equation-for-n-meaningful}) holds.

\subsection{Computational Time}
Figure~\ref{supfig:computation-time-real} shows the computation time of the real ensemble feature selector for Colon, Lymphoma, and Prostate datasets.
Figure~\ref{supfig:computation-time-simulation} shows the computation time of the simulated ensemble feature selector for the same datasets.

\section{Discussion} \label{supsec:discussion}

\subsection{Effect of the Base Models on Parameter Estimation}\label{supsec:effect-of-base-models-on-parameter-estimation}

We investigated how the hyperparameters of the base feature selectors changed the estimation of the simulation parameters $\nmeaningful$ and $\probability$.
Among the various hyperparameters of random forest feature selectors, we considered \textit{mtry}, which is the number of features used for separating a node in a tree.
We paid attention to the \textit{mtry} parameter because it is known to critically affect the performance of random forest predictors~\citep{probst2019hyperparameters}.
Because we evaluated the effect of \textit{mtry} using several datasets with different feature dimensions, we employed the \textit{normalized} \textit{mtry}, which we define as the \textit{mtry} value divided by the feature dimension, as a hyperparameter.
In the implementation of the scikit-learn library, we can configure the normalized \textit{mtry} using the \texttt{max\_features} option.
Previous studies have shown that the square root of the feature dimension is a preferable choice as \textit{mtry} (hence, the normalized \textit{mtry} is the inverse square root of the feature dimension)~\citep{10.1007/978-3-642-02326-2_18}.
Our analysis in Section~\ref{sec:results} also employed this value.

Figure~\ref{supfig:parameter-estimation-with-various-max-features} shows the estimation of the parameters $\nmeaningful$ and $\probability$ for different normalized \textit{mtry} values.
On one hand, when the normalized \textit{mtry} was small ($0.01$), both $\nmeaningful$ and $\probability$ changed as we changed the number of trees $\ntrees$ (Figure~\ref{supfig:parameter-estimation-with-various-max-features}, top).
On the other hand, when the normalized \textit{mtry} was relatively large ($0.1$), the estimated $\nmeaningful$ was approximately $55$ regardless of the number of trees $\ntrees$ (Figure~\ref{supfig:parameter-estimation-with-various-max-features}, bottom).

In the setting of Table~\ref{tab:n-redundant-vs-n-trees}, the normalized \textit{mtry} was small.
Specifically, the normalized \textit{mtry} values were $1/\sqrt{2000}\approx 0.022$, $1/\sqrt{4026}\approx 0.015$, and $1/\sqrt{5966}\approx 0.012$ for Colon, Lymphoma, and Prostate datasets, respectively.
Therefore, the situations were similar to those in Figure~\ref{supfig:parameter-estimation-with-various-max-features} (top).
Accordingly, the estimated $\nmeaningful$ changed as we changed $\ntrees$ in Table~\ref{tab:n-redundant-vs-n-trees}. 

From these observations, we expect that $\nmeaningful$ is highly dependent on the complexity of the feature selectors in the small-normalized-\textit{mtry} regime.

\subsection{Simulation with Almost Constant \texorpdfstring{$\nmeaningful$}{} Value}\label{sec:constant-n-meaningful}

Let us first recap our analyses so far.
On one hand, the stability value of the simulated ensemble feature selector is almost independent of the number of trees $\ntrees$ of the random forest feature selector when the number of weak feature selectors $\nensemble$ is large (Figure~\ref{fig:stability-estimation,stability-actual} (left) in Section~\ref{sec:stability-estimation-of-ensemble-efature-selectors}).
However, it did not occur in the real ensemble feature selectors when the normalized \textit{mtry} was small (Figure~\ref{fig:stability-estimation,stability-actual} (right) in Section~\ref{sec:stability-estimation-of-ensemble-efature-selectors}).
On the other hand, the estimated $\nmeaningful$ significantly depends on $\ntrees$ when the normalized \textit{mtry} is small and vice versa (Figure~\ref{supfig:parameter-estimation-with-various-max-features}).
Considering these two observations, we hypothesize that the discrepancy in behaviors of the real and simulation ensemble feature selectors observed in Figures~\ref{supfig:stability-estimation} and~\ref{supfig:stability-actual} is caused by the difference in the
$\nmeaningful$ value for different $\ntrees$ values (Table~\ref{tab:n-redundant-vs-n-trees}).

To validate this hypothesis, we computed the stability
in the large-normalized-\textit{mtry} regime.
Figure~\ref{supfig:stability-actual-with-large-max-features} shows the stability of the real ensemble feature selectors, with the configuration being the same as that in Figure~\ref{supfig:stability-actual} except that the normalized \textit{mtry} is set to $0.1$ instead of the inverse squared root of the feature dimension.
We observe that the estimated stability value for the large $\nensemble$ is almost the same for all $\ntrees$ values as expected.
This behavior is in contrast to what we observed in Section~\ref{sec:stability-estimation-of-ensemble-efature-selectors} (more specifically, Table~\ref{tab:n-redundant-vs-n-trees} and Figure~\ref{supfig:stability-actual}), where both the estimated $\nmeaningful$ and the stability value of the real ensemble feature selector change as we increase the number of trees $\ntrees$.
These observations support the aforementioned hypothesis.

Next, we focus on the relationship between the parameter $\probability$ and the stability value when $\nmeaningful$ is an almost constant value.
From Figure~\ref{supfig:parameter-estimation-with-various-max-features} (bottom), the parameter $\probability$ is estimated as $\probability= 0.55, 0.65, 0.8, 0.8, 0.8$ for $\ntrees =50, 100, 200, 300, 400, 500$ for Colon dataset, respectively.
In particular, the estimated $p$ value is the same for $\ntrees\geq 200$.
Correspondingly, the stability values for $\ntrees=200$ to $500$ behaves similarly, except that the stability of $\ntrees=500$ is higher than the other $\ntrees$ when $\nensemble=30$ (Figure~\ref{supfig:parameter-estimation-with-various-max-features}, top).
It implies that the parameter $\probability$ explains the behavior of actual ensemble feature selector well.

In conclusion, the simulator accurately reflects the behavior of the stability value of the ensemble feature selectors, using the simulator parameters $\nmeaningful$ and $\probability$.

\section{Limitations and Future Directions} \label{supsec:future_direction}

\subsection{Limitations}

The simulated weak feature selector enables us to estimate the stability of both single and ensemble feature selectors efficiently, as well as to provide a way to treat various feature selector algorithms in a unified manner.
However, it only models how many features a feature selector specifically chooses, and does not
consider which features are preferably chosen by a feature selector.
Therefore, the simulated feature selector does not provide the feature selection results.

Algorithm~\ref{alg:simulated-feature-selector} selects one feature per iteration until it ranks all the features.
However, Theorem~\ref{thm:justification-of-n-meaningful-estimation} only considers its first iteration.
It is assumed that calculating only the first iteration is sufficient because the high-ranked features determine most of the ensemble result.
It would be a good future direction to extend Theorem~\ref{thm:justification-of-n-meaningful-estimation} to the late stage of the algorithm.

Our proposed stability estimation is agnostic to the feature selection algorithm.
However, our experiments only used the random forest feature selector, which is standard in current bioinformatics research.
It is another good future work to investigate the efficacy of the proposed method in other feature selection algorithms such as the Hilbert-Schmidt independence criterion (HSIC) Lasso~\citep{yamada2014high} and other importance score-based selection algorithms such as the permutation importance and the Shapley additive explanations (SHAP)~\citep{NIPS2017_8a20a862}.

\subsection{Future Extensions}

As explained in Section~\ref{sec:model-overview}, we assumed that the ensemble algorithm requires weak feature selectors to rank the features.
Instead, we can extend the proposed method to the case where the ensemble algorithm requires weak feature selectors to output a subset of features as the feature selection results. 
Specifically, we should modify the simulation of $f_m$ to pass $\ntarget$ features with the highest ranks to the ensemble algorithm. 

In this study, we performed a simulated ensemble for all candidate $\probability$ values (and for all $\nmeaningful$ if we verify it) and adopted it as the value of $\probability$ that is closest to the measured value of stability.
However, as we have seen in Section~\ref{sec:p-estimation} and Section~\ref{sec:stability-estimation-of-ensemble-efature-selectors}, the value of stability is a decreasing function with respect to $\probability$ when $\nmeaningful$ is constant.
Therefore, we can estimate the parameter $\probability$ using a binary search.
This reduces the time complexity of the $\probability$ estimation from $O(\nprobability)$ to $O(\log \nprobability)$, where $\nprobability$ is the number of possible values that $\probability$ takes.

\begin{figure}
    \centering
    \includegraphics[width=0.94\linewidth]{figure/prediction_accuracy/colon.jpg}\\
    \includegraphics[width=0.94\linewidth]{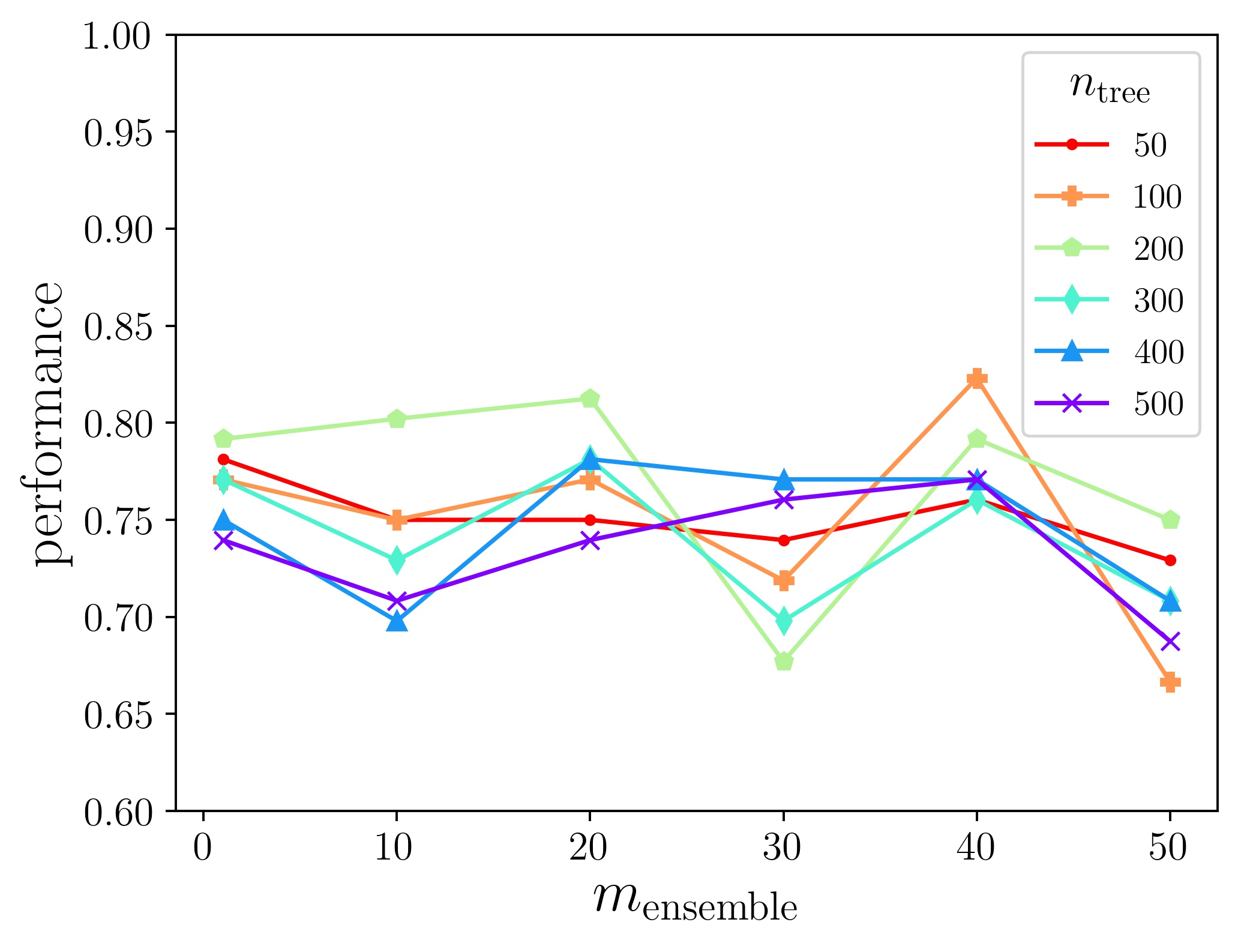}
    \includegraphics[width=0.94\linewidth]{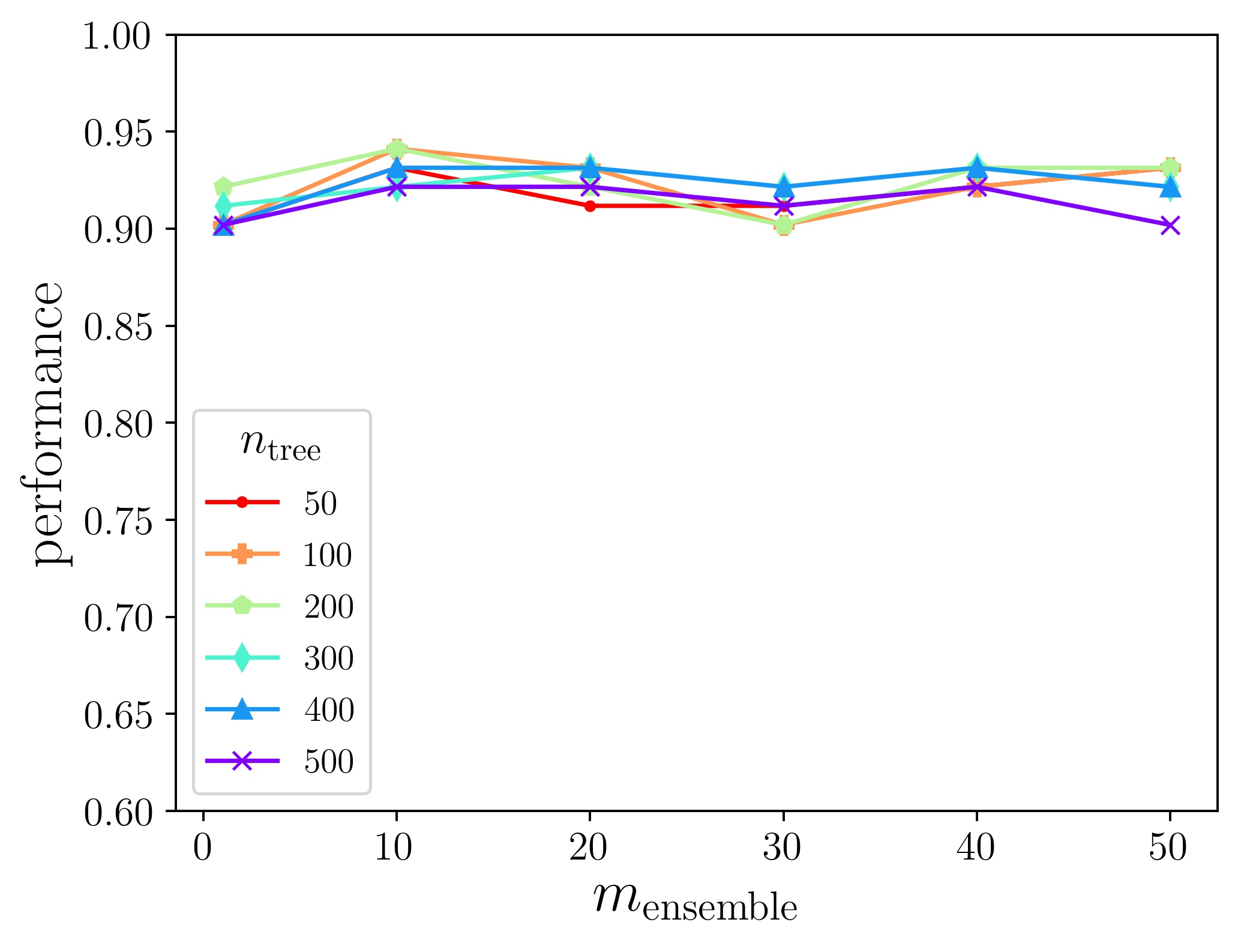}
    \caption{Prediction accuracy for Colon (top), Lymphoma (middle), and Prostate (bottom) datasets.
    }\label{supfig:prediction_accuracy}
\end{figure}
\begin{figure}
    \centering
    \includegraphics[width=0.94\linewidth]{figure/n_redundant_estimation/colon.jpg}\\
    \includegraphics[width=0.94\linewidth]{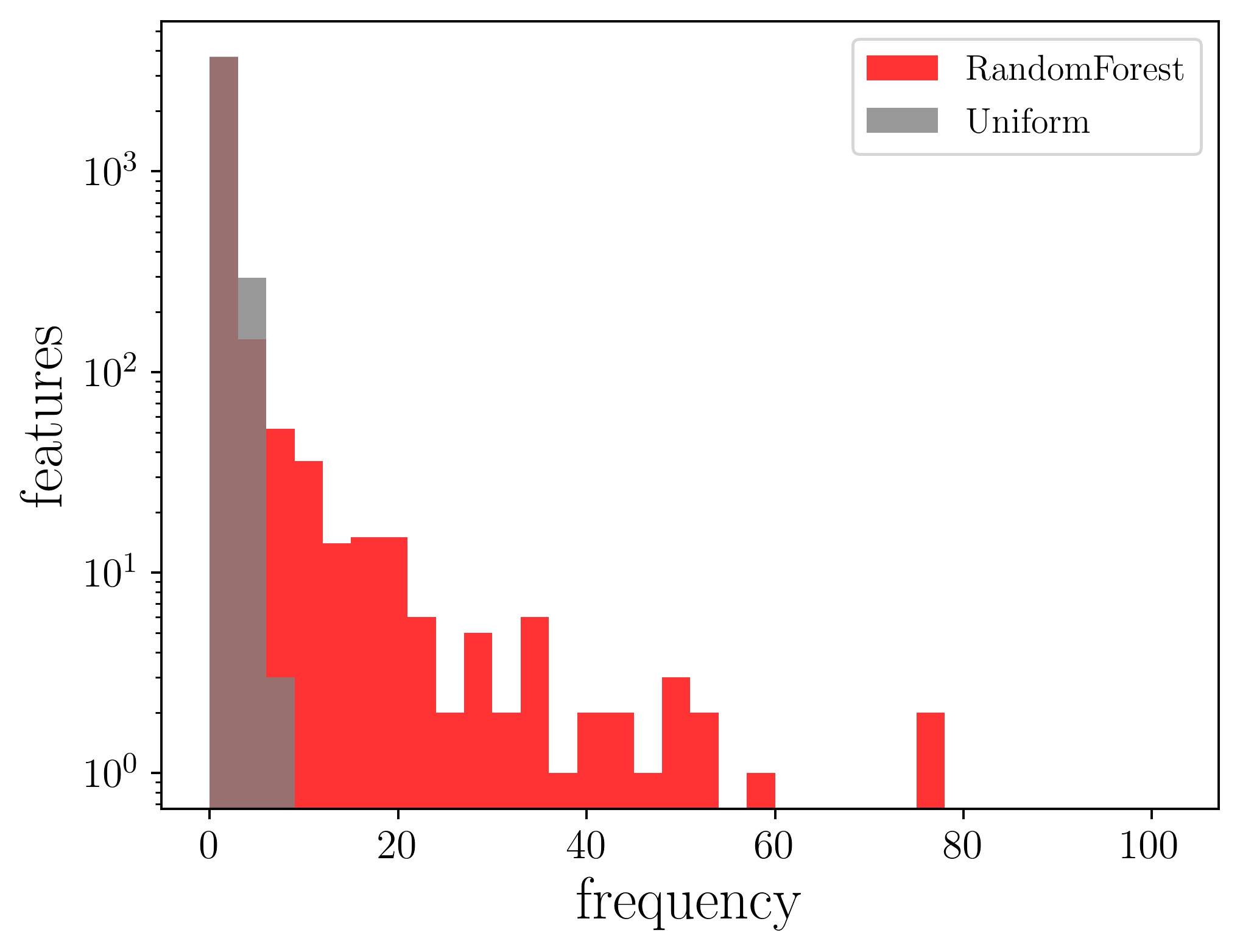}
    \includegraphics[width=0.94\linewidth]{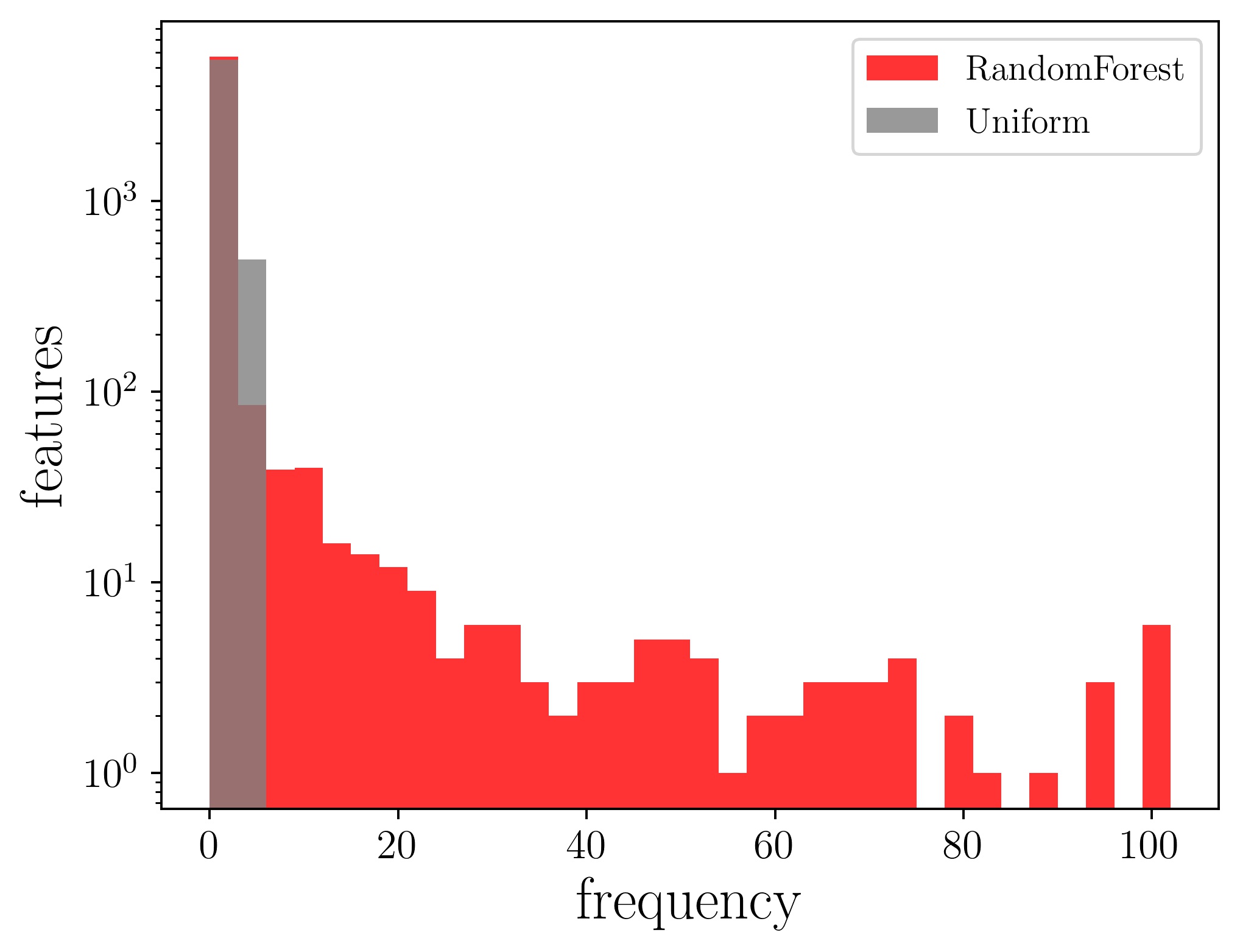}
    \caption{
    Histogram of features' frequencies selected by the random forest feature selector and the uniform feature selector for Colon (top), Lymphoma (middle), and Prostate (bottom) datasets.
    }\label{supfig:n-redundant-estimation}
\end{figure}
\begin{figure}
    \centering
    \includegraphics[width=0.94\linewidth]{figure/p_estimation/colon.jpg}\\
    \includegraphics[width=0.94\linewidth]{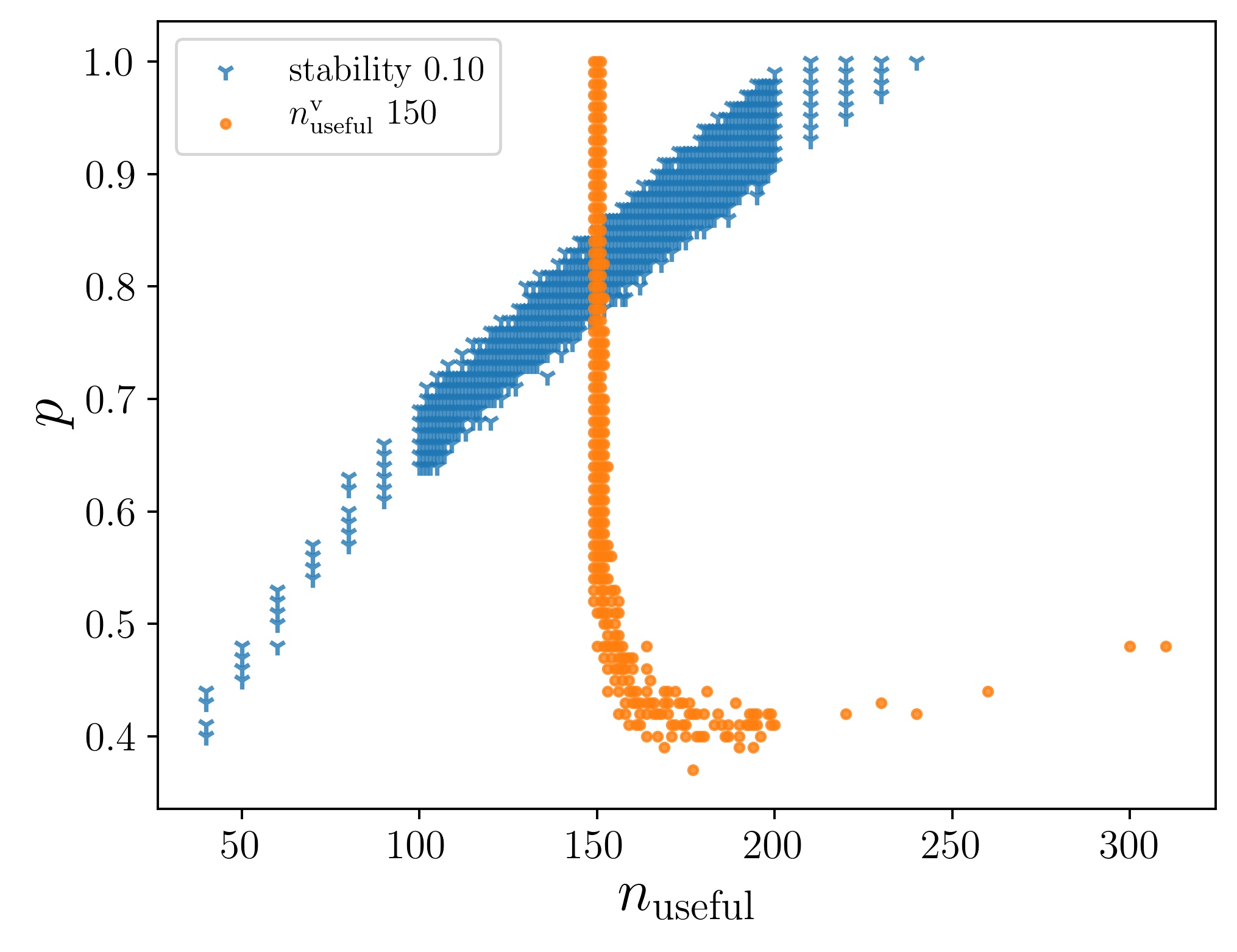}
    \includegraphics[width=0.94\linewidth]{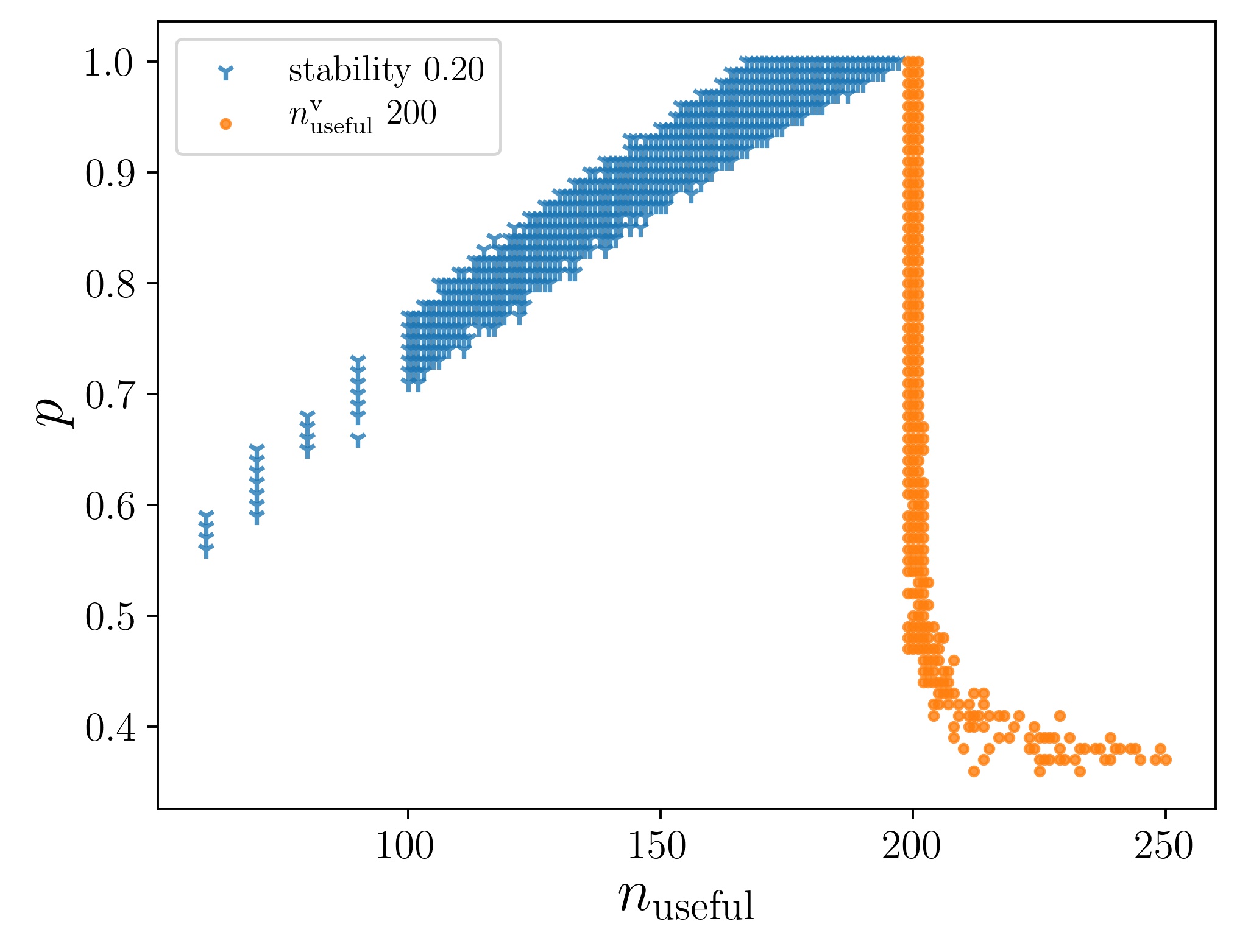}
    \caption{Estimation of the parameter $p$ for the configuration of Colon (top), Lymphoma (middle), and Prostate (bottom) datasets.
    The blue triangle points are parameter pairs $(\nmeaningful, \probability)$ such that the estimated stability value is within $0.1\pm 0.01$ (Colon), $0.1\pm 0.01$ (Lymphoma), and $0.2\pm 0.02$ (Prostate), respectively.
    The orange circle points are parameter pairs whose $\nmeaningfulverification$ value is within $60\pm 2$ (Colon), $150\pm 2$ (Lymphoma), and $150\pm 2$ (Prostate), respectively.
    $\nmeaningful$ ranges over $\nmeaningful=20,\ldots, 810$ (Colon), $\nmeaningful=40, \ldots, 360$ (Lymphoma) and $\nmeaningful=60, \ldots, 260$ (Prostate), respectively.
    $\probability$ ranges over $\probability=0.01, \ldots, 1.00$ for all datasets.}\label{supfig:p-estimation}
\end{figure}

\begin{figure}
    \centering
    \includegraphics[width=0.94\linewidth]{figure/stability_estimation/colon.jpg}\\
    \includegraphics[width=0.94\linewidth]{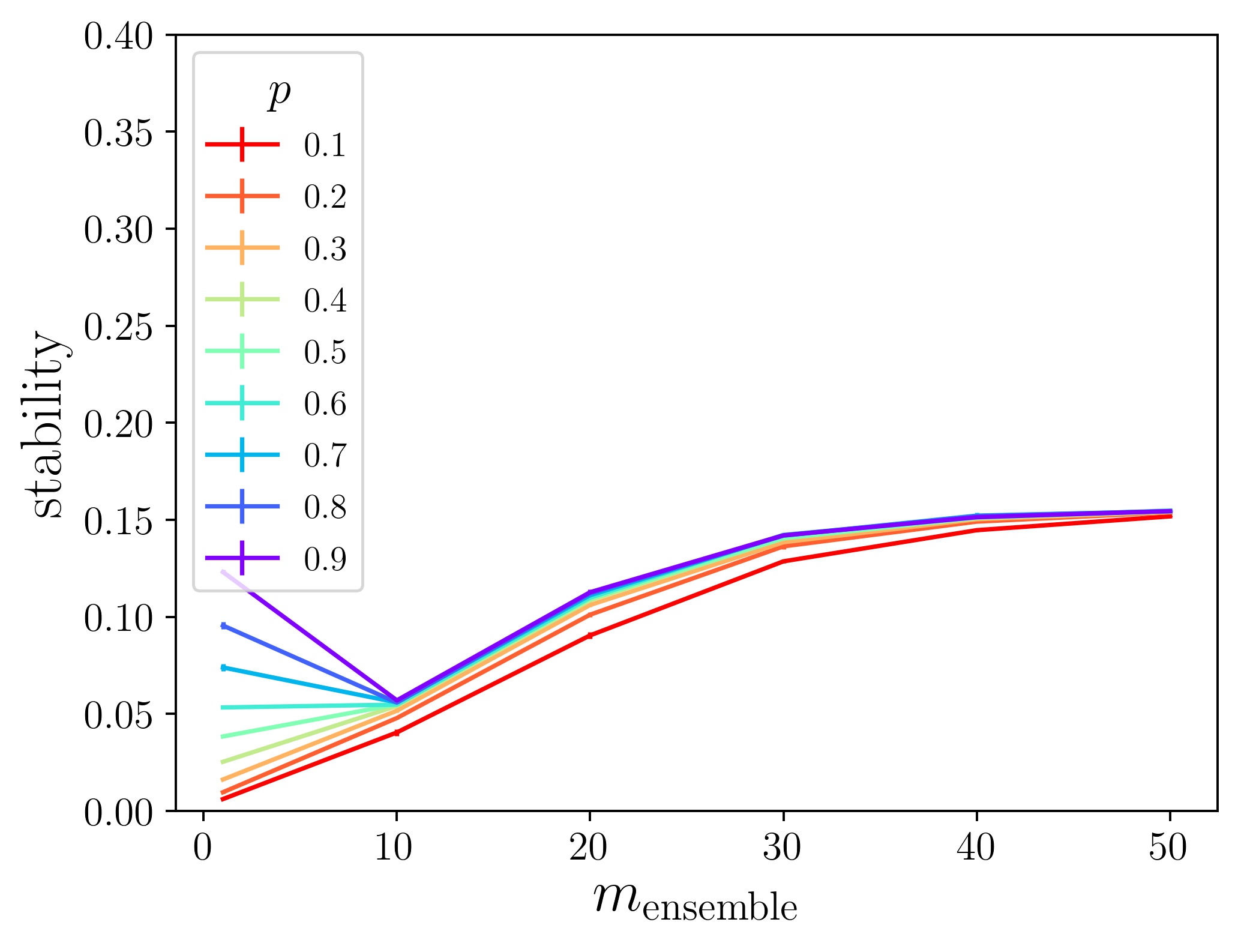}
    \includegraphics[width=0.94\linewidth]{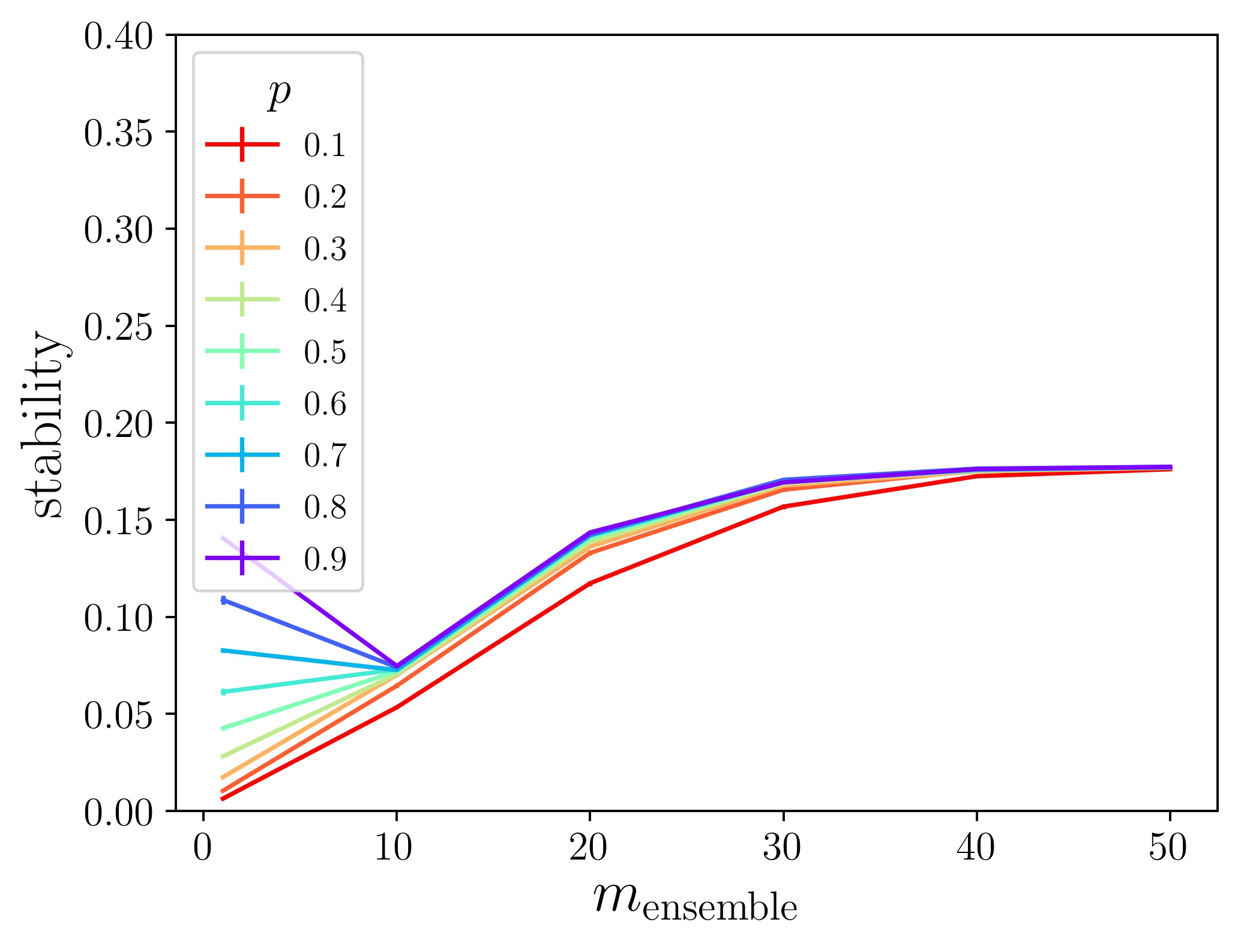}
    \caption{Estimation of stability of the simulated ensemble feature selectors for Colon (top), Lymphoma (middle) and Prostate (bottom) datasets.
    }\label{supfig:stability-estimation}
\end{figure}
\begin{figure}
    \centering
    \includegraphics[width=0.94\linewidth]{figure/stability_actual/colon.jpg}\\
    \includegraphics[width=0.94\linewidth]{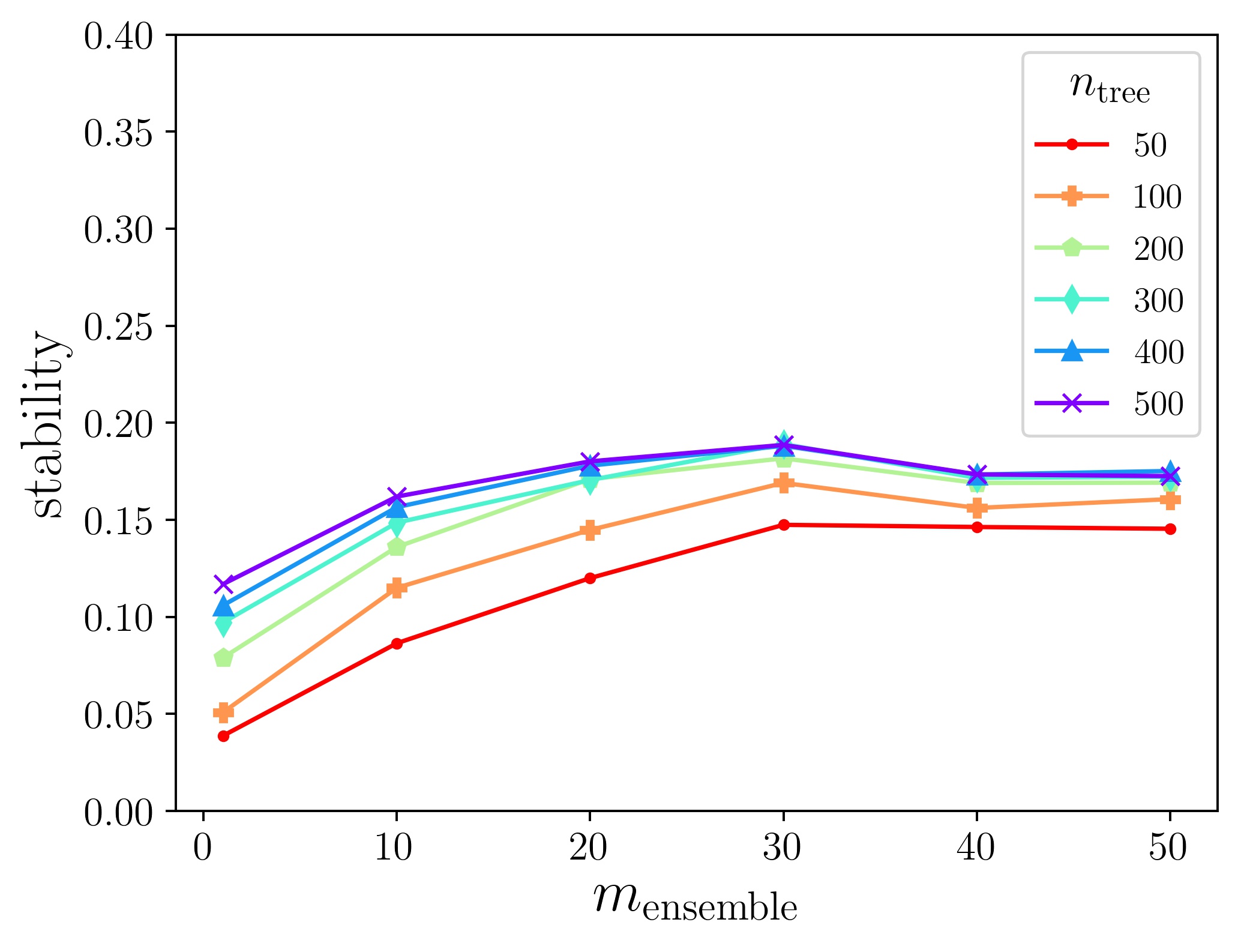}
    \includegraphics[width=0.94\linewidth]{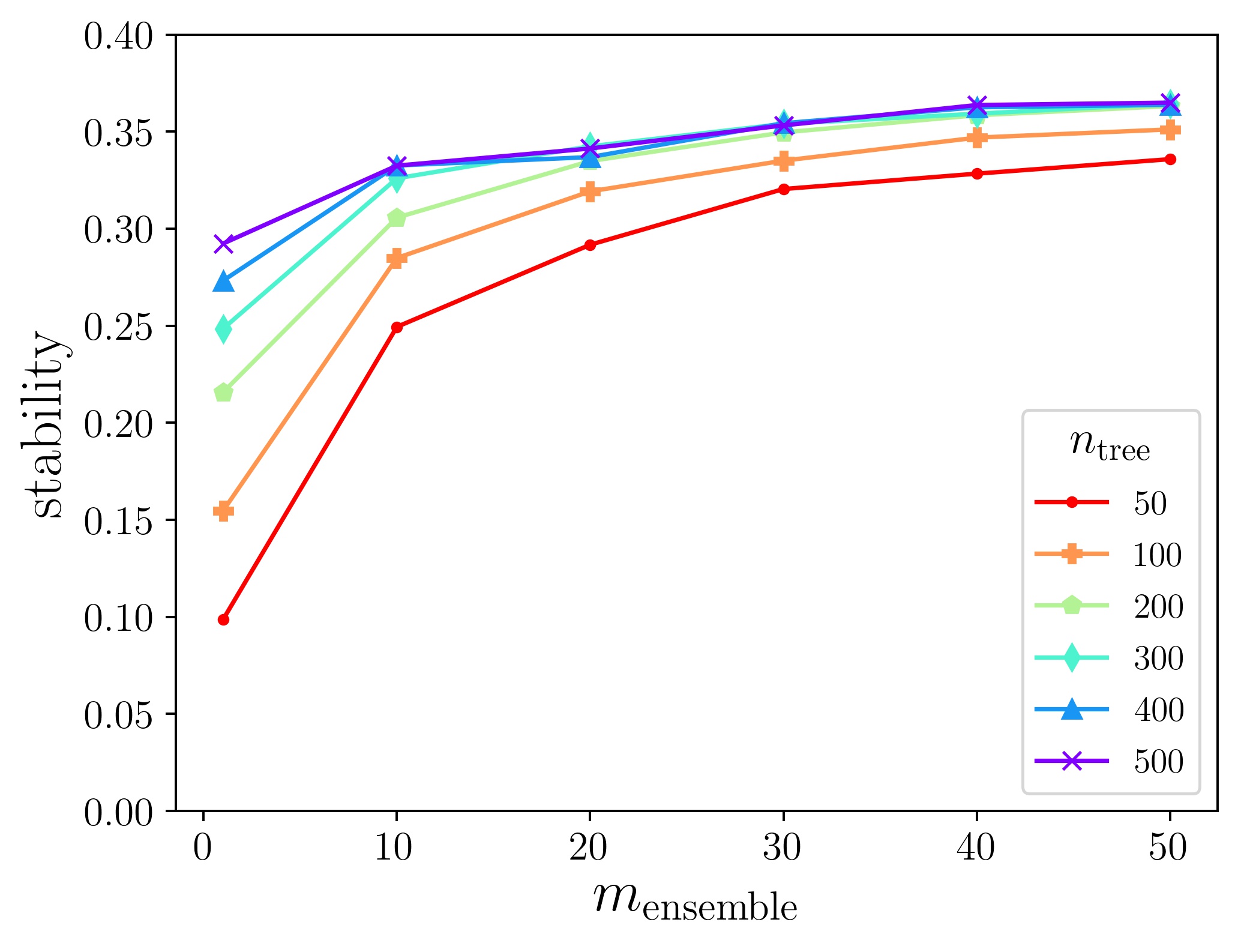}
    \caption{Stability of the real ensemble feature selectors for Colon (top), Lymphoma (middle), and Prostate (bottom) datasets.
    }\label{supfig:stability-actual}
\end{figure}
\begin{figure}
    \includegraphics[width=0.94\linewidth]{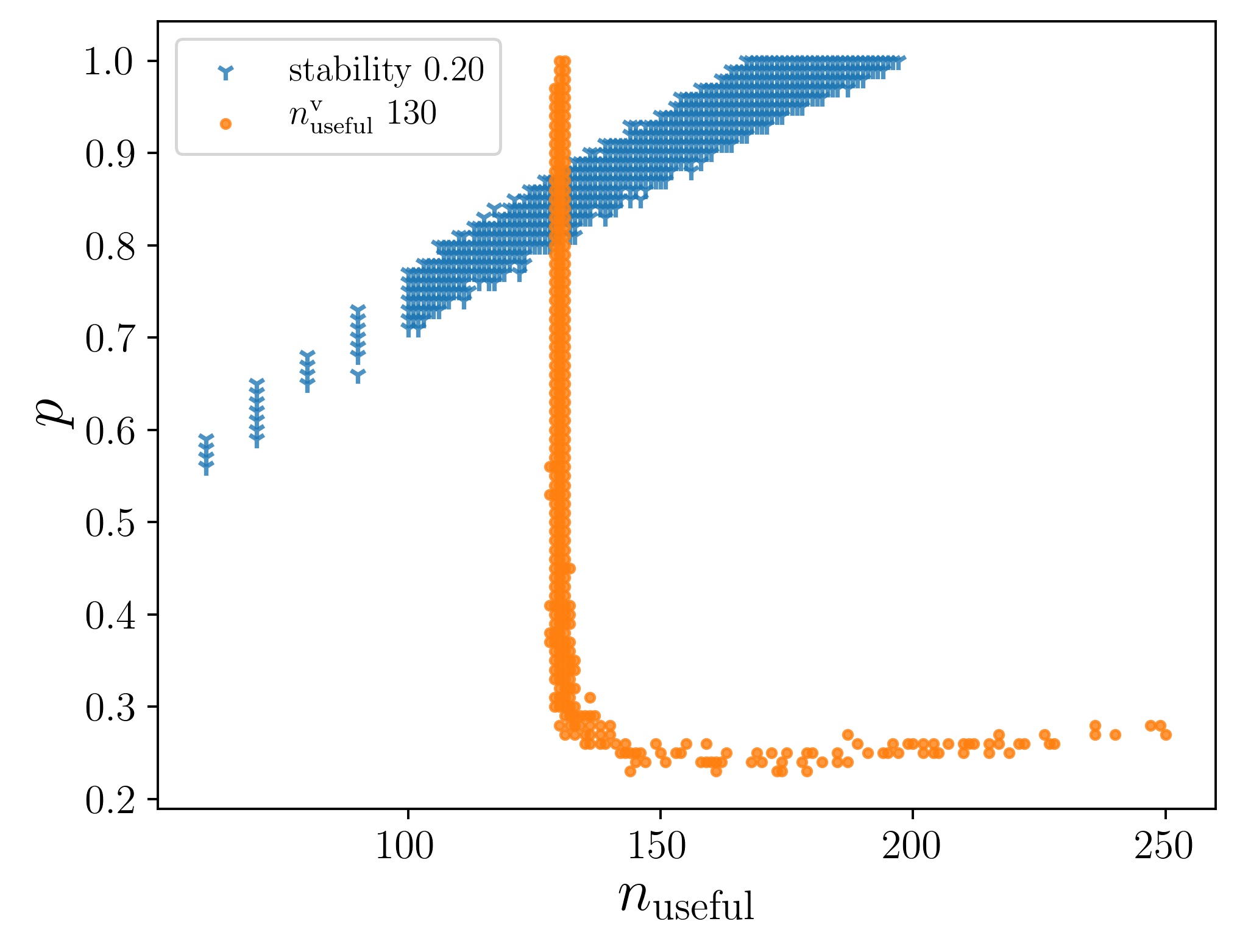}
    \includegraphics[width=0.94\linewidth]{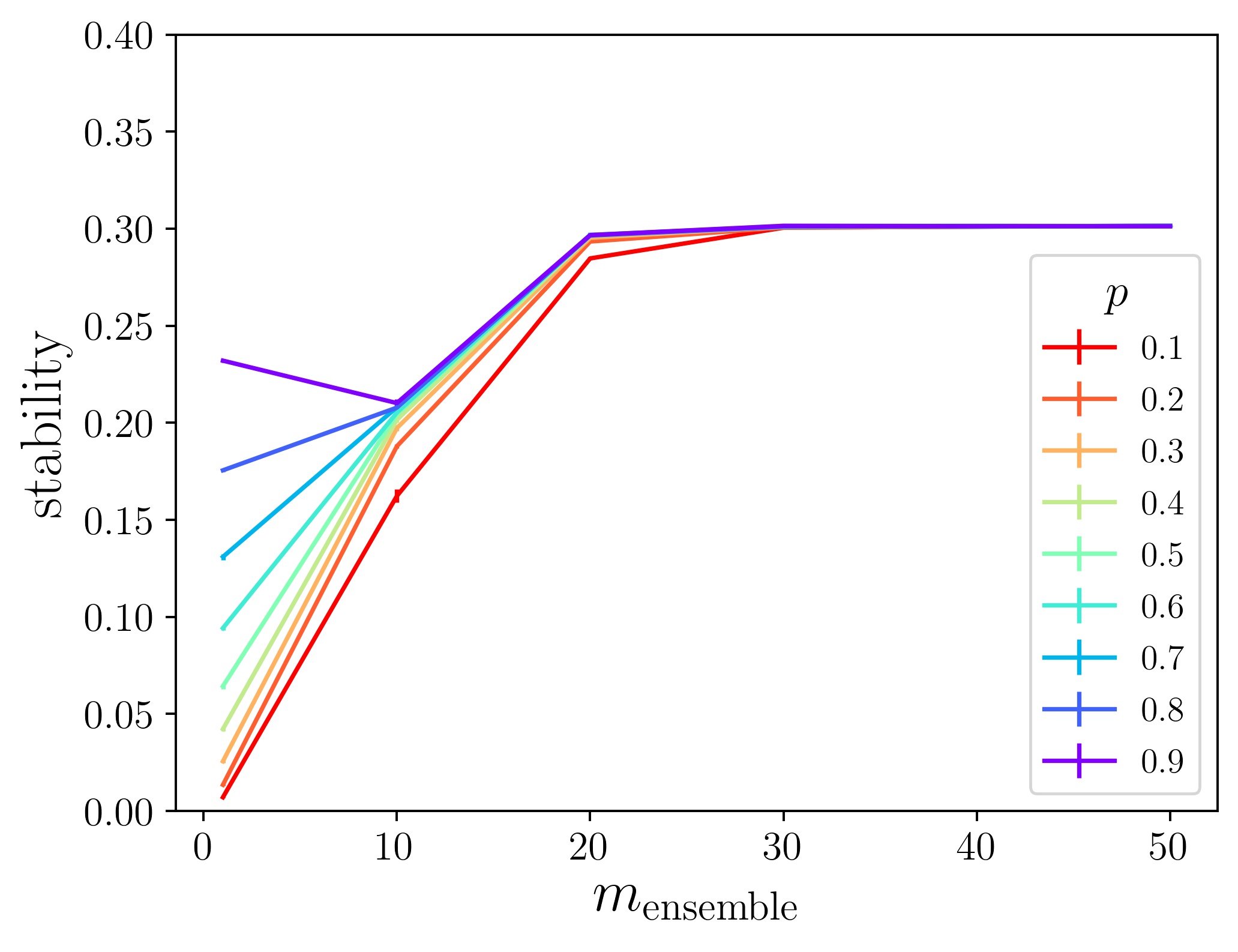}
    \caption{Experiments using the Prostate dataset when the parameter $\nmeaningful$ is set to $130$.
    (Top) Estimation of $\probability$ and verification of $\nmeaningful$. (Bottom) the stability estimation of the simulated ensemble feature selectors.}\label{supfig:prostate-results-with-calibrated-n-redundant}
\end{figure}

\begin{figure}
    \centering
    \includegraphics[width=0.94\linewidth]{figure/computation_time_real/colon.jpg}\\
    \includegraphics[width=0.94\linewidth]{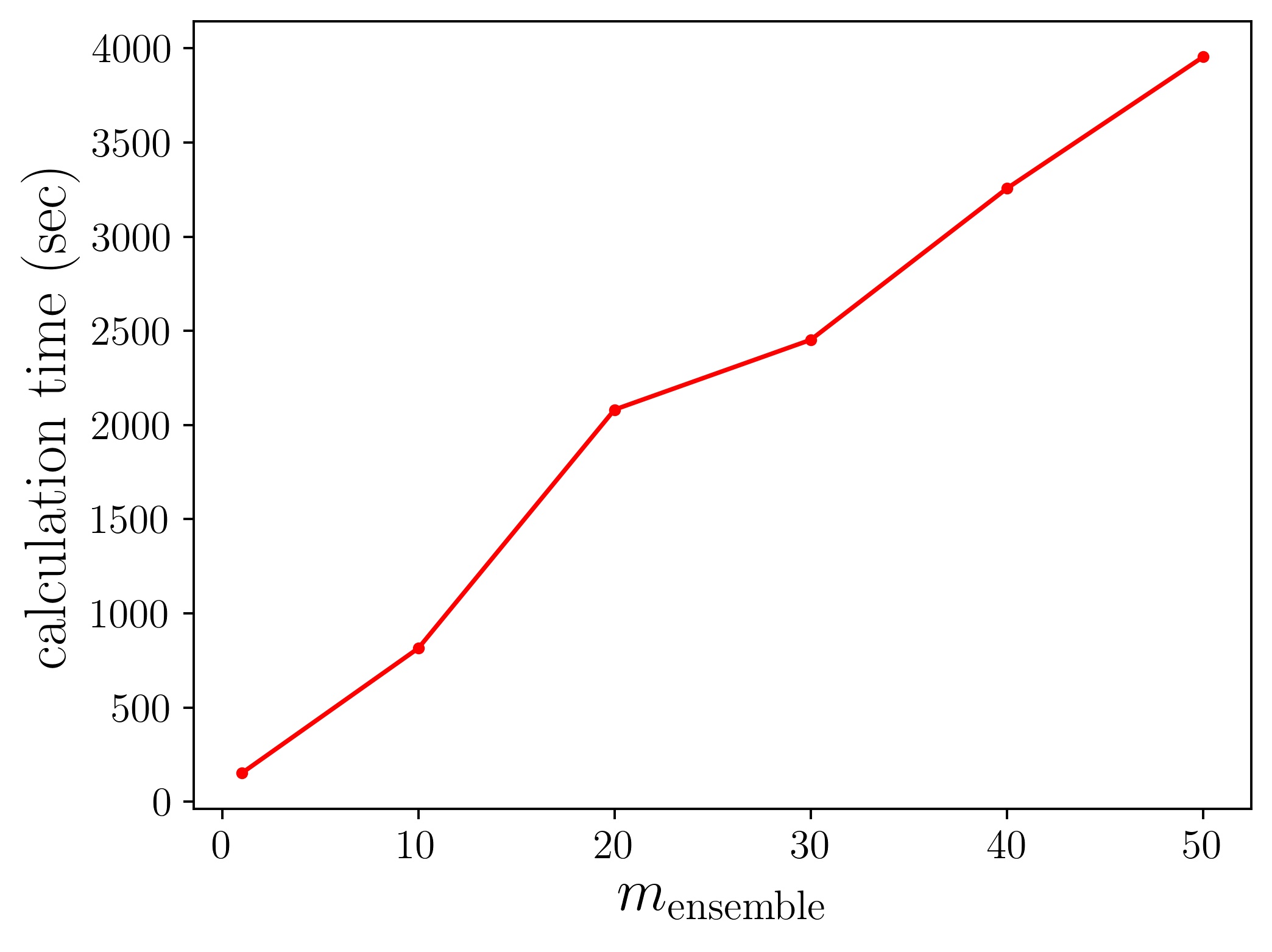}\\
    \includegraphics[width=0.94\linewidth]{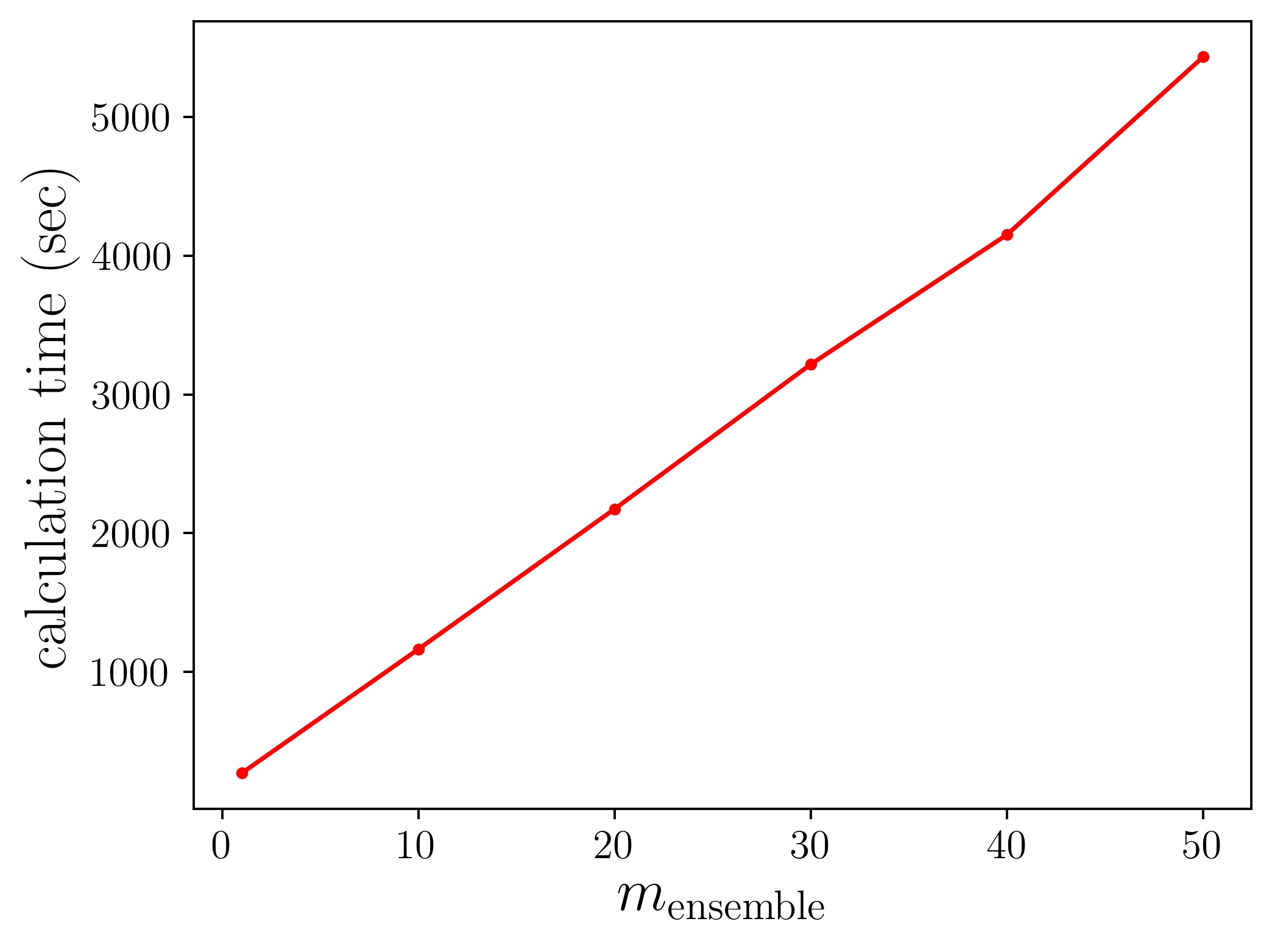}\\
    \caption{Computation time of the real ensemble feature selectors for Colon (top), Lymphoma (middle), and Prostate (bottom) datasets.
    }\label{supfig:computation-time-real}
\end{figure}
\begin{figure}
    \centering
    \includegraphics[width=0.94\linewidth]{figure/computation_time_simulation/colon.jpg}\\
    \includegraphics[width=0.94\linewidth]{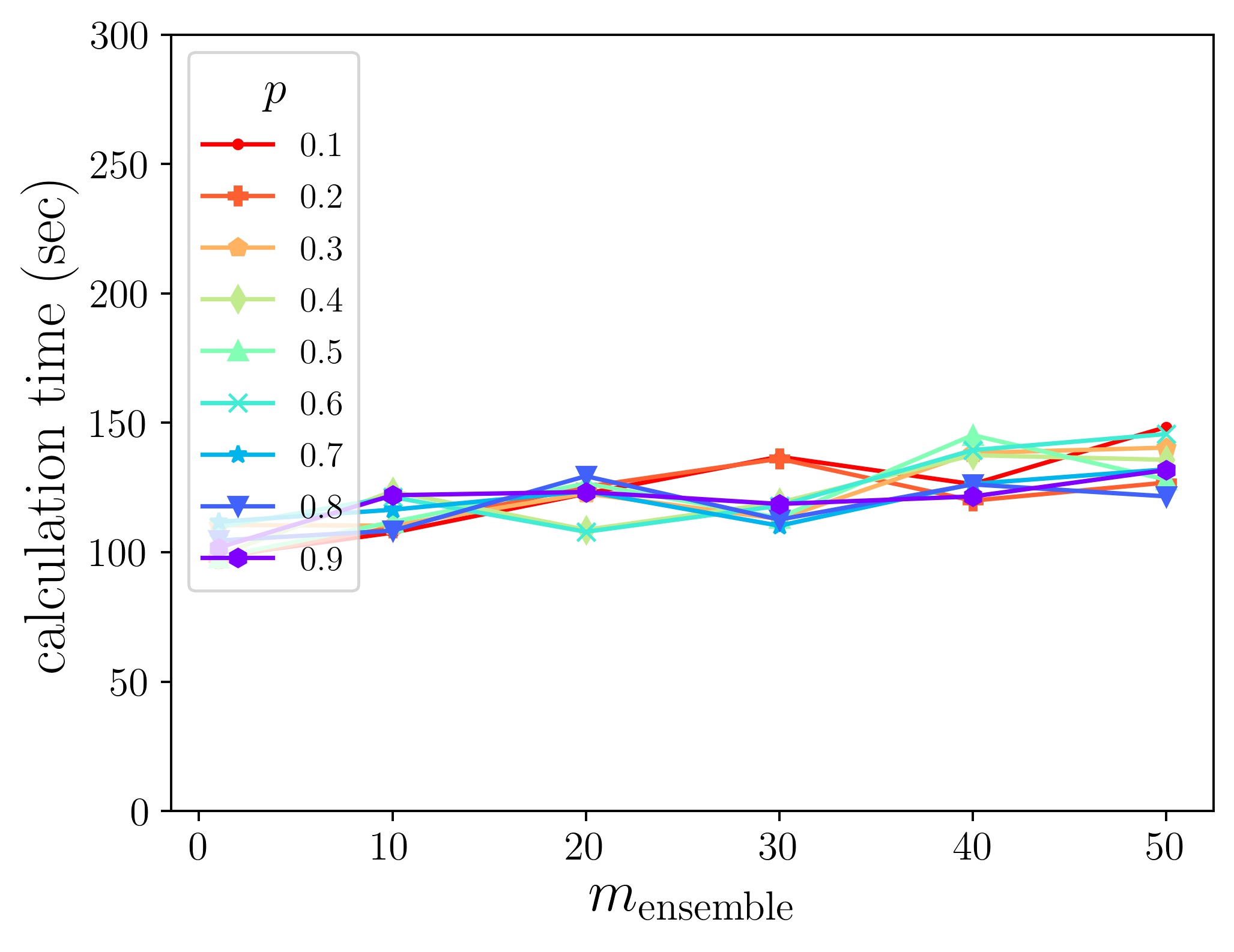}\\
    \includegraphics[width=0.94\linewidth]{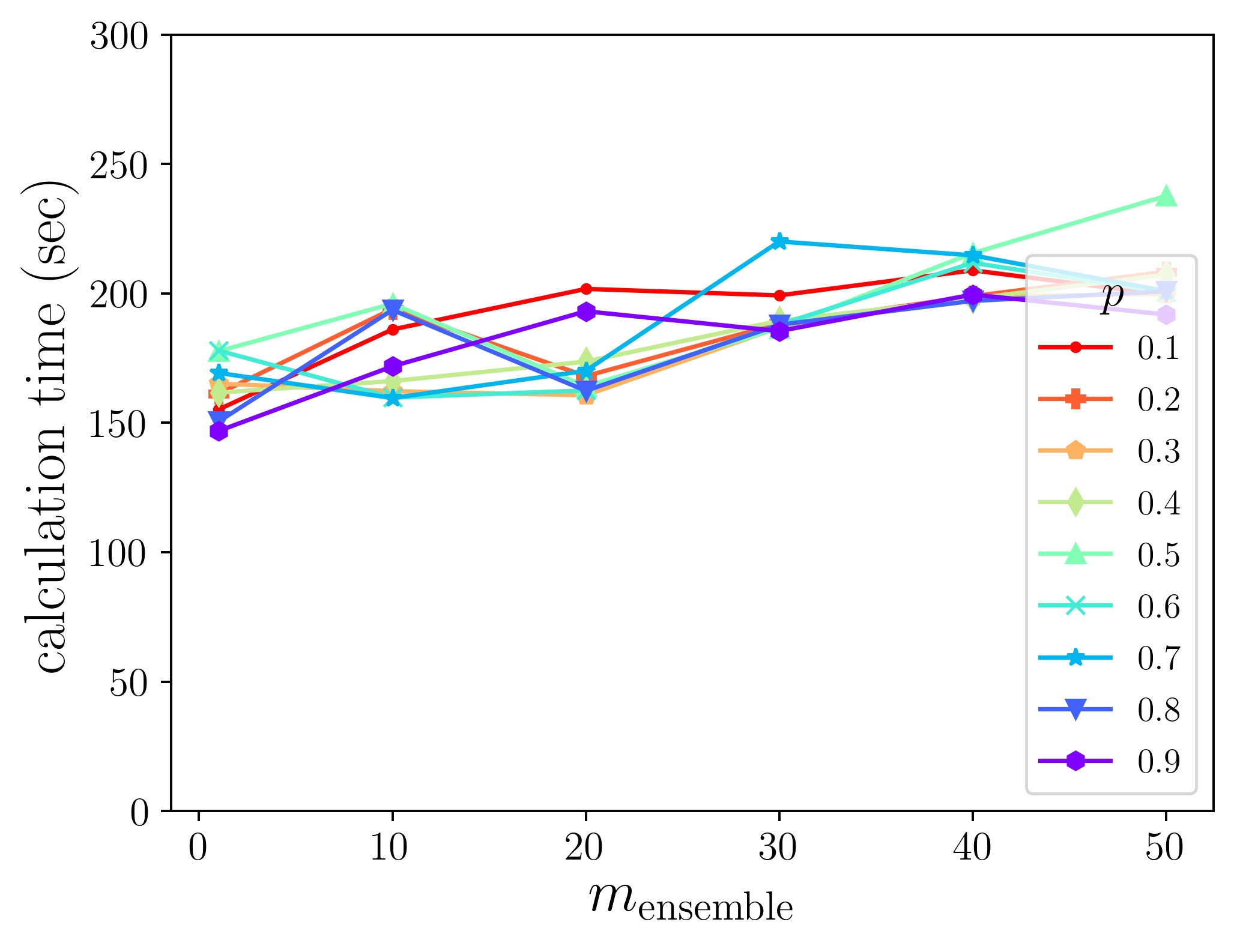}\\  
    \caption{
    Computation time of the simulated ensemble feature selectors for Colon (top), Lymphoma (middle), and Prostate (bottom) datasets.
    }\label{supfig:computation-time-simulation}
\end{figure}
 \begin{figure}[!tpb]
    \centering
    \includegraphics[width=0.94\linewidth]{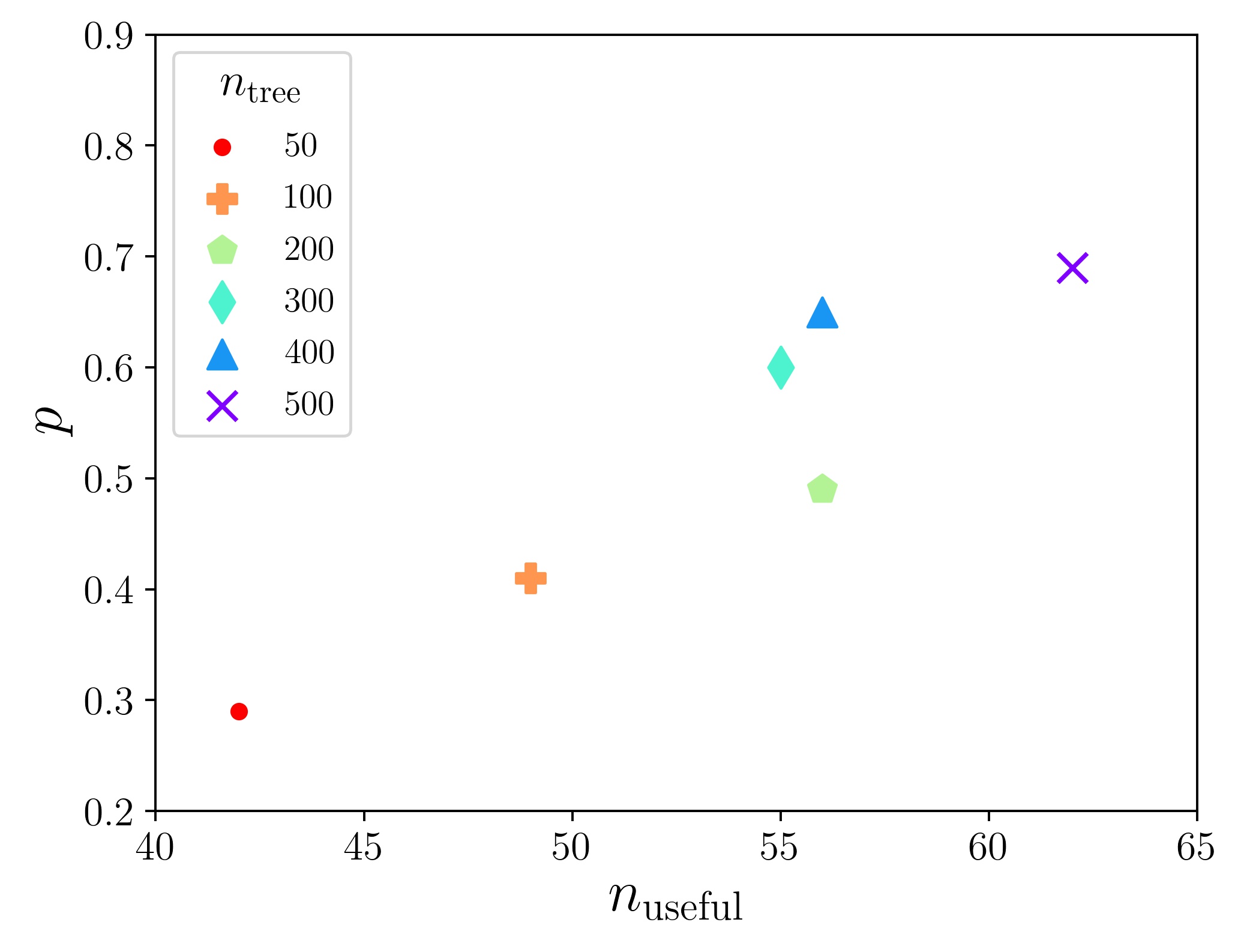}
    \includegraphics[width=0.94\linewidth]{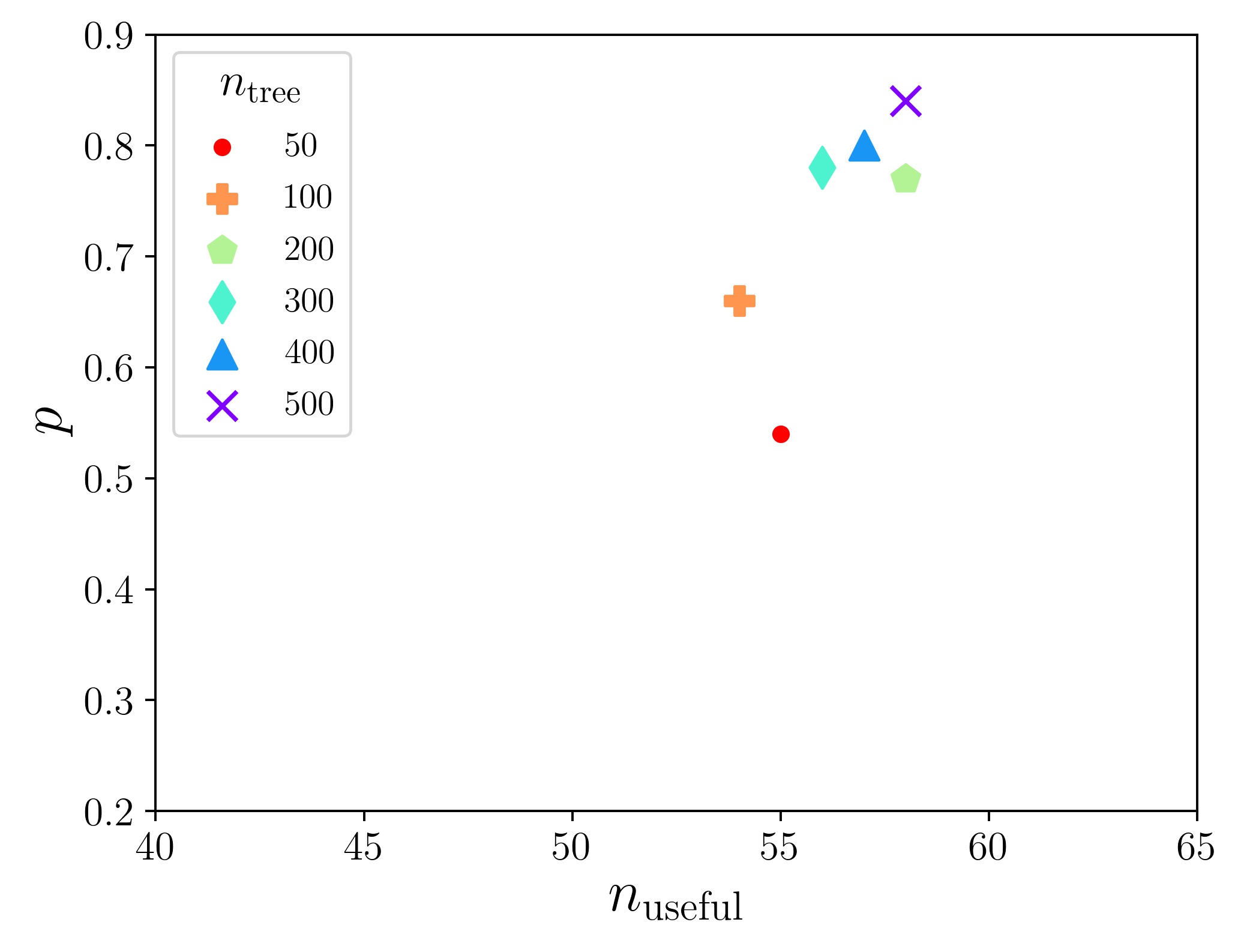}
    \caption{Parameter estimation with different normalized \textit{mtry} values for Colon dataset when normalized \textit{mtry} is $0.01$ (top) and $0.1$ (bottom). }\label{supfig:parameter-estimation-with-various-max-features}
 \end{figure}
\begin{figure}[!tpb]
    \centering
    \includegraphics[width=0.94\linewidth]{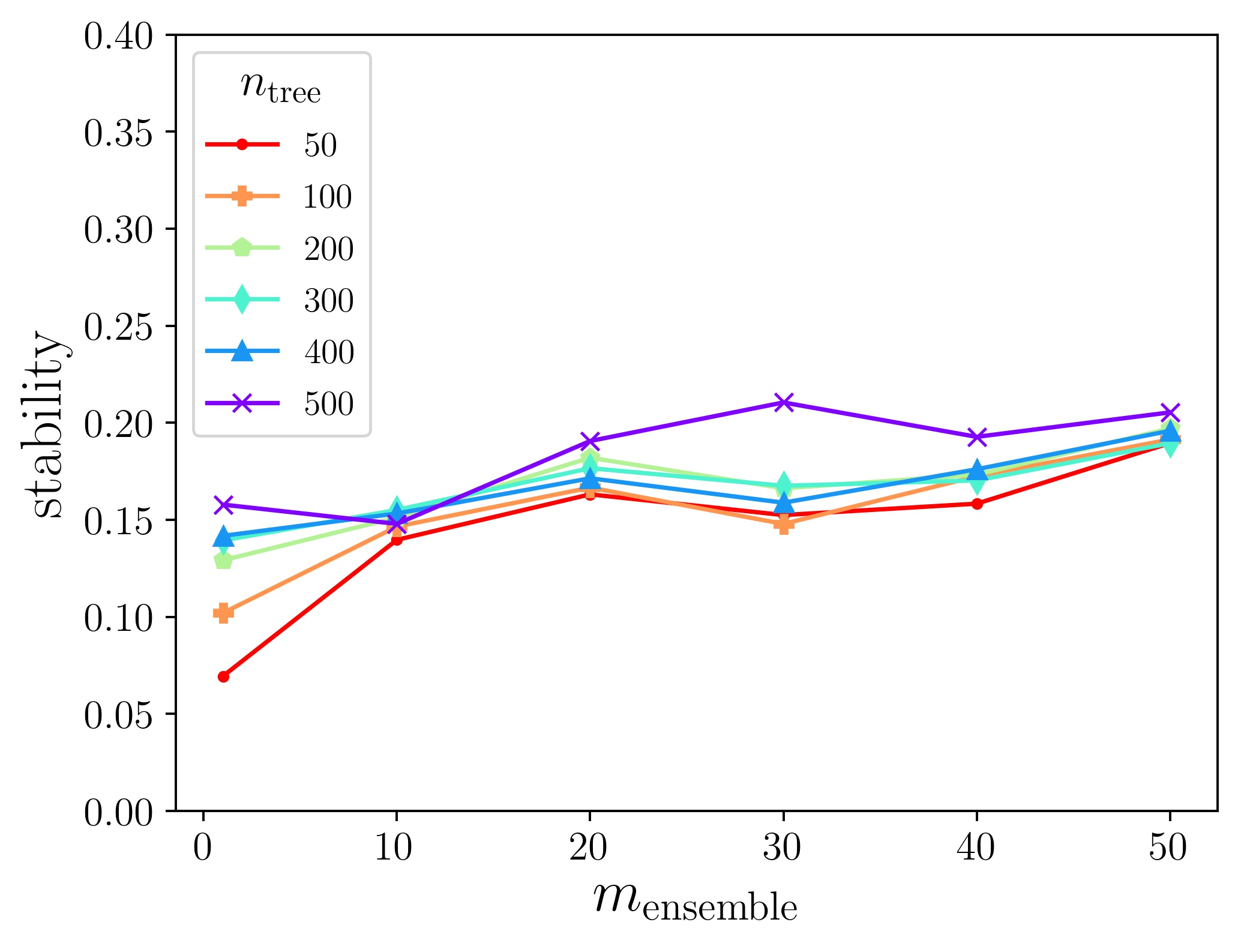}
    \includegraphics[width=0.94\linewidth]{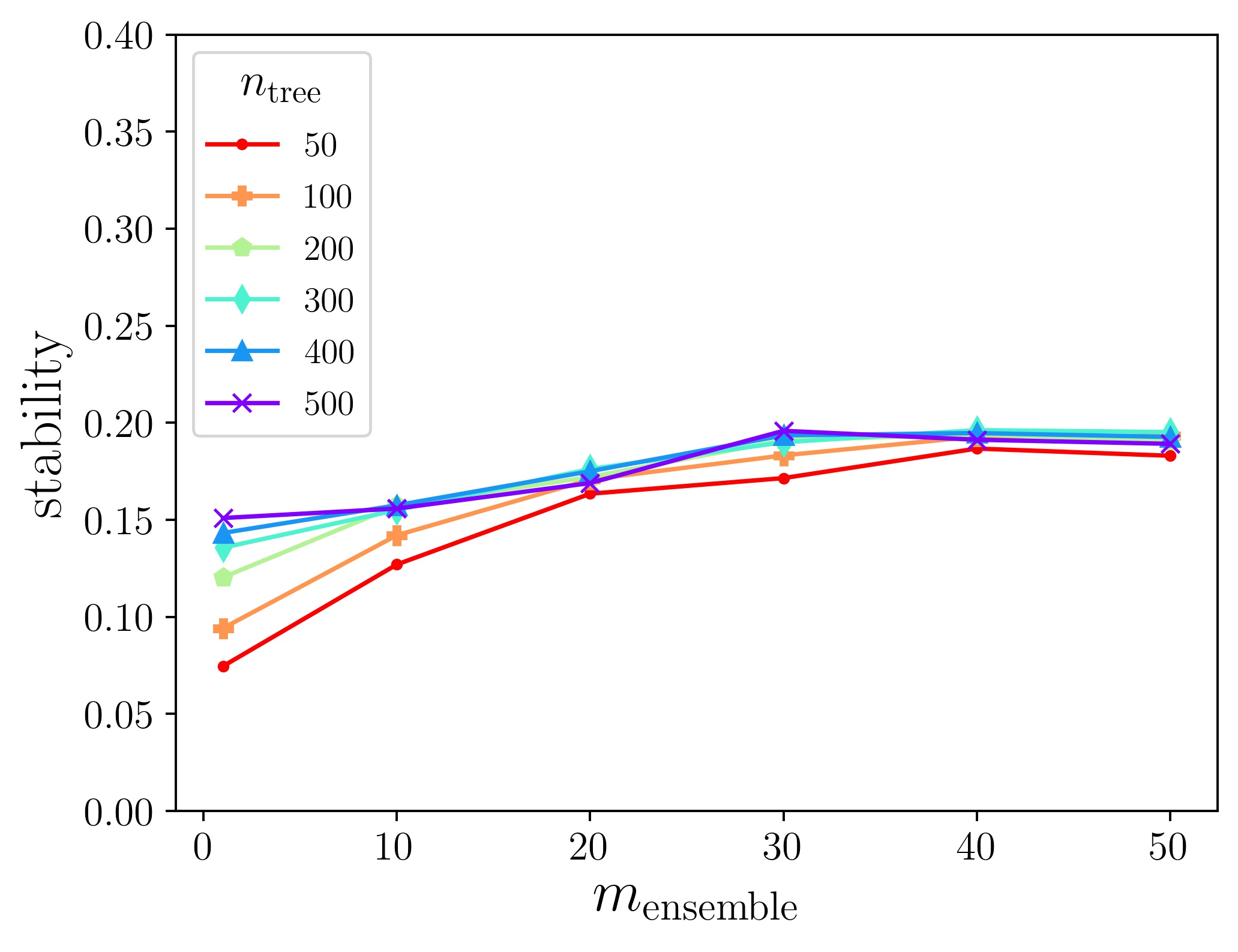}
    \includegraphics[width=0.94\linewidth]{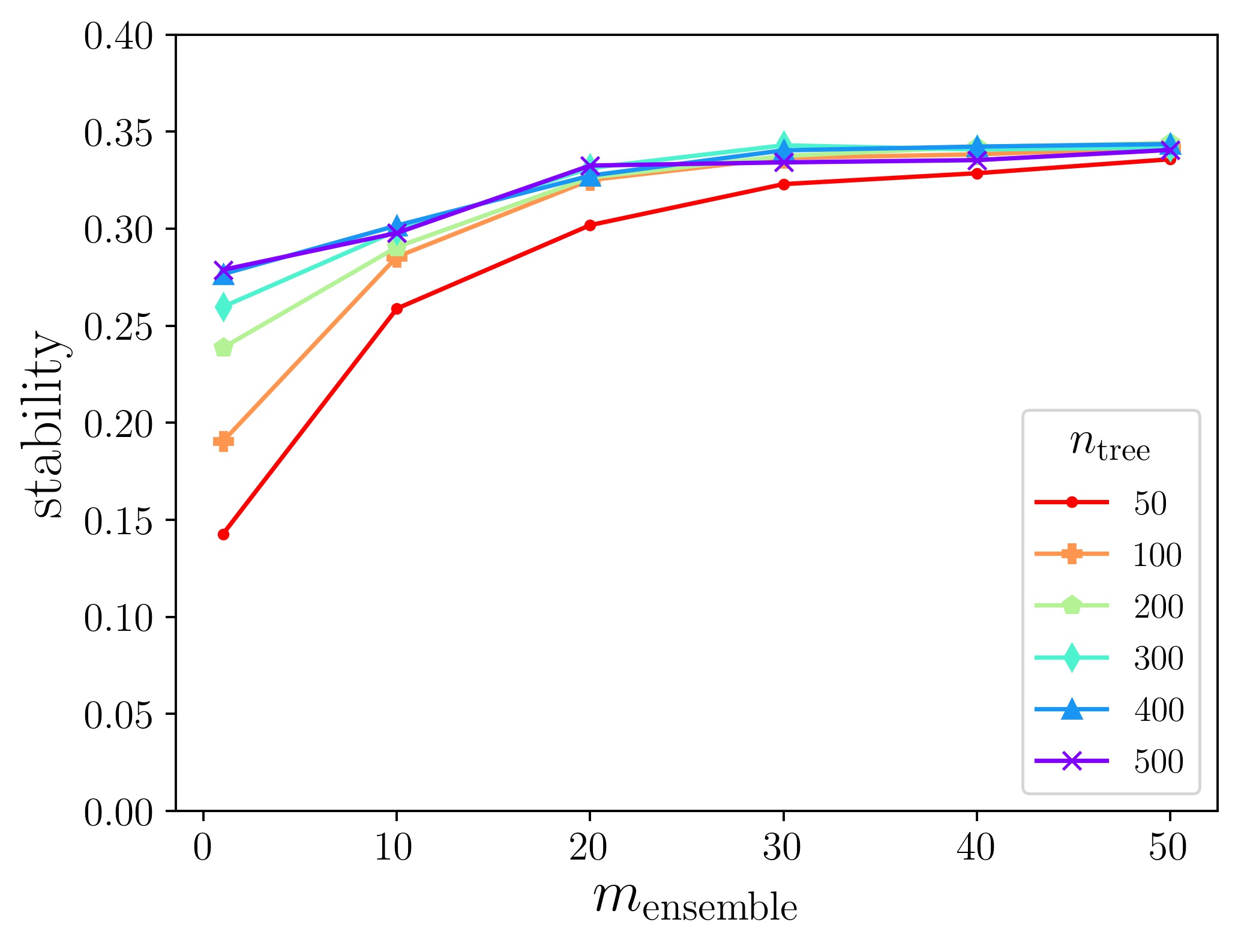}
    \caption{Stability of the real ensemble feature selectors when the normalized mtry is $0.1$ for the Colon (top), Lymphoma (middle), and Prostate (bottom) datasets.
    }\label{supfig:stability-actual-with-large-max-features}
\end{figure}

\begin{table}[p]
  \caption{Notation table. $\ast$ indicates a user-specified variable. \label{tab:notation-table}}
  \label{tab:freq}
  {\begin{tabularx}{\linewidth}{@{}lX@{}}
    \toprule
    Variable & Definition\\
    \midrule
    $D$ & Dataset${}^\ast$. \\
    $f$ & Base feature selector${}^\ast$. \\
    $\mathcal{E}$ & Ensemble algorithm${}^\ast$. \\
    $f_m$ & $m$-th weak feature selector. \\
    $\nensemble$ & Number of weak feature selectors${}^\ast$. \\
    $\nstability$ & Number of feature selectors, single or ensemble, that are used for calculating stability values${}^\ast$.\\
    $S$ & Set of features. Determined by the dataset $D$. \\
    $S_m$ & Feature subset of $S$ from which the $m$-th weak feature selector is likely to choose.\\
    $S’$ & Feature subset of $S$ from which feature selectors are likely to choose.\\
    $\nfeatures$ & $=|S|$. \\
    $\ntarget$ & Number of selected features. $|S_m|$ is set to this value${}^\ast$.\\
    $\nmeaningful$ & $=|S'|$. \\
    $\nmeaningfulverification$ & Verification value for $\nmeaningful$. \\
    $\probability$ & Probability parameter integrating the uncertainty of both datasets and feature selectors.\\
    $\nprobability$ & The number of possible values that the parameter $\probability$ can take${}^\ast$.\\
    $\threshold$ & Threshold determined by the uniform feature selectors for estimating $\nmeaningful$. \\
    $\rankm$ & Ranking of features determined by the $m$-th weak feature selector. Mathematically, $\rankm$ is the permutation of $\{1, \ldots, \nfeatures\}$.\\
  \bottomrule
  \end{tabularx}}{}
\end{table}

\begin{table}[p]
\caption{Threshold $\threshold$ determined by the maximum frequency of the uniform feature selector.\label{tab:threshold-by-random-feature-serlector}}
\resizebox{\columnwidth}{!}{
{\begin{tabular}{@{}llll@{}}\toprule
Dataset & Colon & Lymphoma & Prostate \\\midrule
$\threshold$ & 4.640 $\pm$ 0.636 & 6.029 $\pm$ 0.668 & 6.455 $\pm$ 0.703 \\
\bottomrule
\end{tabular}}{}
}
\end{table}

\end{document}